\documentclass{article} 
\usepackage{iclr2027_conference,times}
\usepackage{amsthm}
\newtheorem{Lemma}{Lemma}

\theoremstyle{definition}
\newtheorem{Definition}{Definition}
\usepackage{amsfonts}

\usepackage{amsmath,amsfonts,bm}

\def\eqref#1{equation~\ref{#1}}

\def\1{\bm{1}}

\def\vtheta{{\bm{\theta}}}
\def\va{{\bm{a}}}
\def\vb{{\bm{b}}}

\def\vr{{\bm{r}}}

\def\vx{{\bm{x}}}

\def\mA{{\bm{A}}}

\def\mG{{\bm{G}}}

\def\mV{{\bm{V}}}
\def\mW{{\bm{W}}}

\def\mLambda{{\bm{\Lambda}}}

\DeclareMathAlphabet{\mathsfit}{\encodingdefault}{\sfdefault}{m}{sl}
\SetMathAlphabet{\mathsfit}{bold}{\encodingdefault}{\sfdefault}{bx}{n}

\usepackage{hyperref}
\usepackage{url}
\usepackage{graphicx}

\usepackage{booktabs}
\usepackage{caption}
\usepackage{subcaption}
\usepackage{amsthm}
\usepackage{listings}
\usepackage{tikz}
\usetikzlibrary{arrows.meta, positioning, patterns}
\usepackage{algorithm}
\usepackage{algpseudocode}
\usepackage{dsfont}

\title{Efficient Message Passing for Partial Differential Equation Priors}

\author{Anna Kazachkova, Leonhard Hennicke, Rainer Schlosser, Ralf Herbrich\\
Hasso Plattner Institute \\
University of Potsdam \\
Potsdam, Germany \\
\texttt{\{anna.kazachkova,leonhard.hennicke,rainer.schlosser}\\
\texttt{ralf.herbrich\}@hpi.de}
}

\iclrfinalcopy 
\begin{document}

\maketitle

\begin{abstract}
Prior information for real-world physical quantities is most elegantly expressed via partial differential equations (PDEs).
In this paper, we propose a novel way to solve PDEs using probabilistic inference on a factor graph. In general, factor graphs provide a natural way to encode prior knowledge into a model as explicit factors; here, this knowledge is provided by a governing PDE, which narrows the solution space, while observed data further shape the posterior over the parameters.
The approximate parameter posterior is inferred using message passing based on moment matching, without posterior sampling or global gradient-based optimization.
We demonstrate our approach on the first-order advection and the second-order semi-linear Fisher-KPP equations, where it achieves predictive accuracy comparable to a standard baseline while providing structured predictive uncertainty. Moreover, the inferred posterior marginal means and uncertainty structure match more closely those obtained using Hamiltonian Monte Carlo than the evaluated mean-field variational inference baseline, while requiring up to 10× less training time in our experiments, with inference speed comparable to variational inference.
\end{abstract}

\section{Introduction}
\label{section:introduction}

Partial differential equations (PDEs) often describe physical processes in space and time. In many practical problems, a known parameterized equation is accompanied by a set of observations that provide additional information and further constrain the problem. Combining these observations with the PDE forms a machine learning problem with multiple objectives called \textit{physics-informed learning}.
The idea of solving this multiobjective problem by training a neural network was introduced in \cite{raissi2019physics} and termed a physics-informed neural network. (PINN). There, the solution of an equation is represented by a neural network trained using both observed solution values and the governing PDE. We adopt the same setup in this work, but instead of solving it as a global optimization problem, we formulate the \textit{PINN as a factor graph}, where the neural network learns from the observed data probabilistically, and the governing \textit{PDE acts as prior physical knowledge} that constrains the solution space.

While effective for many problems, the standard PINN formulation yields only a point estimate of the network parameters; our work instead adopts a Bayesian formulation and infers a posterior distribution over these parameters.
Existing Bayesian methods for PINNs \citep{yang2021b,zong2025randomized,raj2025deep} typically perform posterior inference through Hamiltonian Monte Carlo (HMC; \citealp{duane1987hybrid}) or gradient-based variational inference (VI; \citealp{jordan1999introduction}). Sampling can be computationally expensive, while practical variational approximations may differ substantially from sampling-based references, especially in predictive uncertainty \citep{blei2017variational}.  

We approach this inference problem from a different angle and use \textit{message passing}: by propagating information directly through the factor graph, we obtain an efficient, gradient-descent- and sampling-free framework for approximate Bayesian inference over the PDE solution.
Our goal is not to establish message passing as a general replacement for Bayesian PINNs, but to test whether this inference perspective is viable in a physics-informed setting. We use HMC as a sampling-based reference and demonstrate that our approach yields similar results in prediction accuracy and posterior parameter estimates, while being more time-efficient.

\paragraph{Problem formulation}
In this work, we solve the PDE \textit{forward problem}, i.e., predicting the future states of a parameterized physical system. For the shallow neural representation used throughout this work, the trainable parameters \(\vtheta:=\mW\) are treated as random variables with prior \(p(\vtheta)\), and the learning objective is to infer their posterior distribution from both observed data and physical constraints. 
Given observations
\(\mathcal{D}=\{(\vx_n,y_n)\}_{n=1}^{N}\) and PDE constraints imposed at
collocation points \(\mathcal{C}=\{\vx_m\}_{m=1}^{M}\) and modeled through PDE residual variable $R$, the target distribution is
\begin{equation}
\label{eq:bayesian_posterior}
    p(\vtheta \mid \mathcal{D},\mathcal{C})
    \propto
    p(\vtheta)\,
    p(\mathcal{D}\mid\vtheta)\,
    p(R = 0 \mid \mathcal{C}, \vtheta).
\end{equation}

Our message passing algorithm approximates this posterior with a Gaussian belief $q(\vtheta)$:
\begin{equation}
\label{eq:posterior-approx}
    q(\vtheta)
    \approx
    p(\vtheta\mid\mathcal D,\mathcal C).
\end{equation}

In this work, we focus on first-order linear PDE constraints of this form for a given coefficient vector~$\vb$:
\begin{equation}
\label{eq:first_order_pde}
    \vb^\top \nabla_{\vx} f_U(\vx;\vtheta) = 0.
\end{equation}

\paragraph{Our contributions}
We aim to answer the following question: Can approximate Bayesian inference on a factor graph solve PDEs from observed data and encoded prior physical knowledge using message passing?
Our contributions are summarized as follows:
\begin{itemize}
    \item \textbf{Probabilistic factor-graph formulation:}
    We introduce a novel perspective on PDE solving as probabilistic inference on a factor graph (Section~\ref{section:approach}). In this graph, a neural network is trained to satisfy the observed data, while incorporated as prior physical knowledge PDE narrows the solution space. Deterministic relations between variables are encoded as Dirac-delta distributions, and therefore messages are propagated with the Direct Message Approximation (DMA; \citealp{herbrich2026dma}) framework; more details are in Definition \ref{definition:dma} in Section \ref{subsection:fg}. We derive the required factor messages for the first-order linear advection equation and the second-order semi-linear Fisher-KPP with a nonlinear reaction term.

    \item \textbf{Competitive PDE solution accuracy with uncertainty estimates:}
    In our experiments, the message passing approach recovers PDE solutions that closely match the exact solutions, while additionally providing uncertainty estimates (Section~\ref{subsection:solution-fidelity}). We further observe that the predicted variance exhibits structured behavior across the evaluated PDE dynamics.
    
    \item \textbf{Posterior marginal inference close to HMC at substantially lower computational cost:}
    Compared with HMC and VI, our approach recovers similar posterior marginal means while avoiding posterior sampling and gradient-based variational optimization (Section~\ref{subsection:bayesian_comparison}). It is consistently faster than HMC and remains computationally competitive with VI, while its posterior uncertainty is closer in structure to the HMC reference.
\end{itemize}

\section{Preliminaries}

To frame our approach, we first revisit the fundamental concepts and introduce notation for approximate message passing on factor graphs.

\subsection{Probabilistic inference on factor graphs}
\label{subsection:fg}
\begin{Definition}A {\textit{factor graph}}
is a bipartite graph with two types of nodes - representing variables and functions of these variables. In the probabilistic inference context, it represents a joint density of random variables that can be computed as the product of all factors. Hidden variables of this probabilistic model are inferred with a \textit{message passing} algorithm. 
\end{Definition}

In this work, we use a \textit{Gaussian approximation} for all information flow on the factor graph, i.e., all variables and messages are represented by Gaussian distributions. When arithmetically possible and computationally efficient, we compute the resulting moments exactly; in most cases, however, we approximate the resulting distribution by a Gaussian via moment-matching. As stated in Equation \ref{eq:posterior-approx}, our inferred posterior is approximate - we elaborate on it further in Appendix \ref{appendix:theory}.

\begin{Definition}{\textit{Message passing}} generally refers to a framework in which nodes in a complex system exchange simple local messages to perform inference \citep{mackay2003information}. In the probabilistic inference context, these nodes are random variables, factors encode statistical dependencies between them, and messages are distributions that communicate information about these variables. In particular, we use the \textit{sum-product algorithm} \citep{kschischang2001factor}, in which information is propagated through the factor graph by passing messages between variable and factor nodes: variable nodes collect information from their neighboring factors, while factor nodes combine this information according to the local dependency they represent and pass the resulting information to neighboring variables. For a general introduction to the sum-product algorithm and Gaussian arithmetic, refer to Appendix~\ref{appendix:sum-product}. Throughout this work, we use the notation $m_{X_k\rightarrow f_i}(x_k)$ and $m_{f_i\rightarrow X_k}(x_k)$ to denote messages passed from a variable node $X_k$ to a factor node $f_i$ and from a factor node $f_i$ to a variable node $X_k$, respectively.
\end{Definition}

\begin{Definition}{\textit{Direct Message Approximation}} (DMA; \citealp{herbrich2026dma}) is a framework that extends the sum-product algorithm. Instead of approximating the marginals, Gaussian messages from factors to variable, $m_{f_i\rightarrow X_k}(x_k)$ are directly approximated via moment-matching. In our PDE factor graph, most relations between variables are modeled as Dirac-delta factors, which make DMA the only admissible framework for such message passing; we therefore use this approach by default throughout our work. More details on the background and its relation to other frameworks, such as Expectation Propagation \citep{minka2001family} and Variational Message Passing \citep{winn2005variational}, can be found in the original paper. 
\label{definition:dma}
\end{Definition}

\paragraph{Bayesian neural network}
We define a single-layer Bayesian neural network with a single nonlinear output unit, trainable parameters $\vtheta:=\mW$, an activation function $g(\cdot)$, and a basis feature vector function $\phi(\cdot)$. The input $\vx\in\mathbb{R}^d$ is mapped to the basis feature vector $\phi(\vx)\in\mathbb{R}^{P}$, where $P:=\dim\phi(\vx)$, and $\mW\in\mathbb{R}^{P}$. In our Bayesian formulation, the weights, preactivation, and network output are represented as Gaussian random variables, with
\begin{equation*}
W_k \sim \mathcal N\!\left(W_k;\cdot,\cdot \right), \quad
z(\vx) = \mW^\top\phi(\vx), \quad
f_U(\vx;\vtheta) = g\!\left(z(\vx)\right)
\end{equation*}

where if the input \(\vx\) is clear from context, we write \(z:=z(\vx)\). We denote the mean and variance of the current Gaussian belief for $W_k$ by $\mu_k$ and $\sigma_k^2$, respectively. Note that we define this shallow architecture as it is sufficient for our current experiments, and use the same architecture for all baselines to ensure a consistent comparison. For the one-dimensional spatiotemporal PDEs considered further, we write \(\vx=(s,t)\), where \(s\) denotes the spatial coordinate and \(t\) denotes time. For brevity, we sometimes write
\(f(\vx;\vtheta):=f_U(\vx;\vtheta)\) when no ambiguity arises.

\subsection{Differentiable activation function}
\label{subsection:activation}

The employed neural network must be differentiable up to the order of the differential equation being considered; i.e., for an $ n$-th-order differential equation, the required partial derivatives of the network output up to order $n$ must exist and be non-constant. Therefore, we have to choose the activation function $g(\cdot)$ and the feature map $\phi(\cdot)$ to be $n$-times differentiable. We provide more details on constructing $\phi(\cdot)$ in Appendix~\ref{appendix:basis}.

In previous work on PINNs, the hyperbolic tangent is commonly used as an activation function \citep{raissi2019physics} - it is infinitely differentiable and analytic on all of $\mathbb{R}$ and therefore satisfies the requirement for a PDE of order \(n\) that \(g\) be \(n\) times differentiable almost everywhere, except on a null set. However, our message passing framework introduces an additional condition on the activation function \(g\): we also require that for Gaussian \(X\) the moments of \(g^{(k)}(X)\) and, where defined, its inverse \(\left(g^{(k)}\right)^{-1}(X)\), $k \in \{0,..., n\}$ can be computed or accurately approximated; when the inverse is not uniquely defined, the corresponding backward message can instead be obtained through a suitable local approximation.
This condition allows the activation and its derivatives to be propagated through the factor graph via moment matching.

The exponential function is a natural choice: it is infinitely differentiable, and exponentiated Gaussian variables have closed-form moments, as they follow a log-normal distribution. To keep these properties while allowing the activation function to also take negative values, we define

\begin{equation}
\label{eq:activation}
    g(x) =
    \begin{cases}
        e^{\beta x}-1, & x>0, \\
        -\alpha\left(e^{-\beta x}-1\right), & x\le0,
    \end{cases}
\end{equation}
where \(0<\alpha<1\) and \(\beta>0\). The function is continuous and
piecewise infinitely differentiable; its only non-differentiable point is \(x=0\), which has zero probability under a non-degenerate Gaussian preactivation. The corresponding Gaussian message approximations are derived in Appendix~\ref{appendix:activation}, using Lemma~\ref{lemma:conditional_lognormal}.

Note that we do not claim any advantage of our message passing approach from the introduction of this activation function -- in all baseline comparisons, we use the same activation function. Rather, we introduce it as a smooth exponential analog of the leaky ReLU, suitable for our message passing approach. In particular, for inputs close to zero, its two exponential branches are approximately linear, $e^{\beta x}-1 \approx \beta x$ for $x>0$ and $-\alpha(e^{-\beta x}-1) \approx \alpha\beta x$ for $x\leq 0$, giving different slopes on the positive and negative sides as in the leaky ReLU.

\section{Our factor graph}
\label{section:approach}
For clarity, we split the factor graph into two subgraphs: the \textit{PDE subgraph} and the \textit{data subgraph}. The PDE graph enforces the PDE at collocation points $\mathcal{C}$ from the domain with a PDE factor $f_R$ (Section \ref{subsection:pde_graph}), while the data graph incorporates observed function values and updates weights with $f_U$ to learn $\mathcal{D}$ (Section \ref{subsection:data_graph}). In our PINN setting, these observations are given only at the initial condition, $t=0$. These subgraphs are shown in Figure \ref{fig:fg_single_layer}: the neural network parameters $W$ are shared between the two subgraphs and get updated on both sides. All deterministic relations between variables are modeled as Dirac-delta distributions, $\delta(\cdot)$. For a fixed neural network architecture, the data subgraph remains unchanged regardless of the PDE being solved; only the PDE factor $f_R$ in the PDE subgraph needs to be adapted to the specific PDE. In Section \ref{subsection:schedule}, we elaborate further on the schedule for message updates.


\begin{figure}
    \centering
\includegraphics[width=0.9\linewidth]{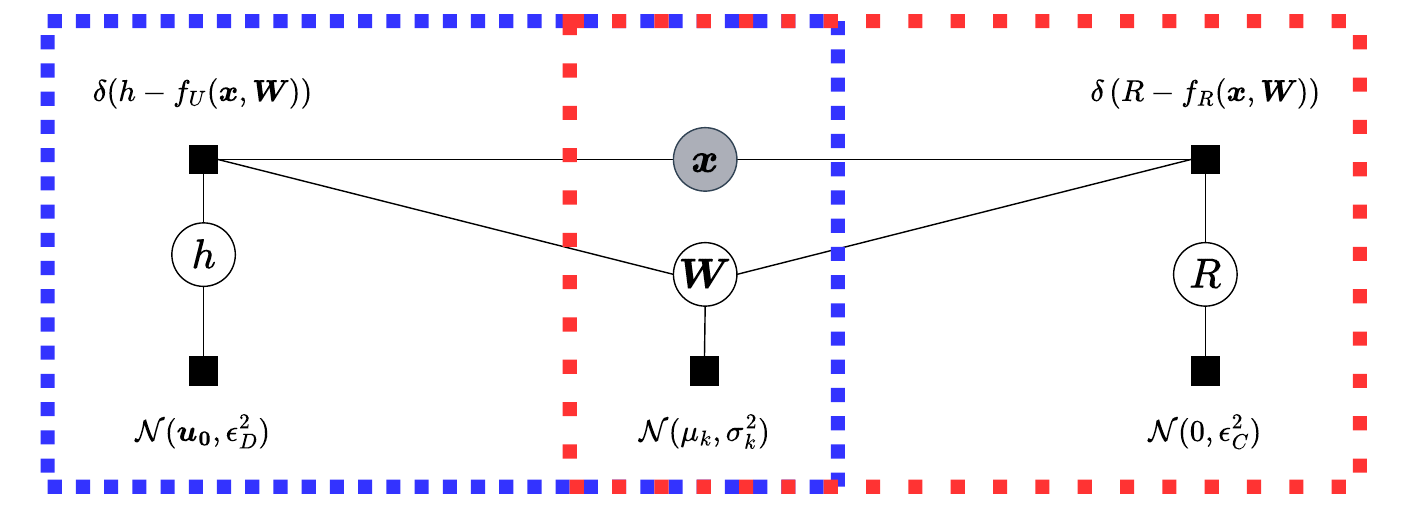}
    \caption{A generic factor graph for PDEs. White circle nodes refer to the hidden variables, the gray node to the observed data $\vx$, and black squares to the functions (factors). \textcolor{blue}{Left subgraph (blue rectangle)}: data subgraph, where $f_U(\vx, \mW)$ models a neural network that is trained to match the label data - only initial condition in our case. \textcolor{red}{Right subgraph (red rectangle)}: PDE subgraph, that enforces a PDE on a subset of collocation points $\mathcal{C}$ with $f_R$.}
    \label{fig:fg_single_layer}
\end{figure}

\subsection{PDE subgraph}
\label{subsection:pde_graph}

The PDE subgraph incorporates the governing equation through a
\textit{residual factor} \(f_R(\vx,\mW)\). At each collocation point \(\vx\in\mathcal C\), the residual $R(\vx,\mW)\sim\mathcal N\!\left(0,\epsilon_C^2\right))$ measures how well the network output satisfies the PDE and is constrained around zero.
Thus, the residual factor connects the PDE constraint at a collocation point to the weight variables. Its specific form is determined by the operator of the desired PDE, while the generic message-passing procedure remains unchanged. We do not enforce the PDE residual $R$ to be exactly zero; instead, we use
$R\sim\mathcal N\!\left(\cdot;0,\epsilon_C^2\right)$; an additional ablation study on the PDE residual prior in Appendix \ref{appendix:pde-residual} demonstrates that if we tighten the space too much, the solution can exhibit the failure mode.

\paragraph{Forward message}
Propagate the current weight distributions through the PDE operator and moment-match the resulting residual distribution,
\begin{equation*}
m_{f_R \rightarrow R}(R)
\approx
\mathcal N\!\left(
    R;\,
    \mathbb E_{\mW}[f_R(\mW,\vx)],\,
    \operatorname{Var}_{\mW}(f_R(\mW,\vx))\right).
\end{equation*}

\paragraph{Backward message}
Propagate the residual $R$ and all remaining
weights $\mW_{\setminus k}$ through the PDE operator 
and moment-match the resulting distribution,

\begin{equation*}
m_{f_R \rightarrow W_k}(W_k)
    \approx
    \mathcal N\!\left(
        W_k;\,
        \mathbb E_{R,\mW_{\setminus k}}[f_R(\mW,\vx)],\,
        \operatorname{Var}_{R,\mW_{\setminus k}}(f_R(\mW,\vx))
    \right). 
\end{equation*}
Note that in real equations, it is usually not feasible to solve this equation for $W_k$ in a closed form; for this reason, we use a first-order Taylor expansion around parameter means. As an illustration, the residual of the advection equation is

\begin{equation}
\label{eq:advection_residual}
c(\mW,\vx)
:=
\vb^\top\nabla_{\vx}f(\vx;\vtheta)
=
g'(z)\mW^\top
\sum_{i=1}^{d} b_i\frac{\partial\phi(\vx)}{\partial x_i}.
\end{equation}

Then $f_R(\mW,\vx) = \delta\left(R - c(\mW,\vx)\right)$ denotes the factor that connects the weight variables \(\{W_k\}_k\) and the residual variable \(R\), as shown in Figure \ref{fig:fg_single_layer}. The derivations for the forward $m_{W \rightarrow f_R}(\mW)$ and the backward $m_{f_R \rightarrow W}$ are shown in the Appendix \ref{appendix:advection}. The same principle can be reused for other PDEs once the corresponding residual factor and its approximate messages have been derived; Appendix~\ref{appendix:kpp} provides the Fisher-KPP equation as a second example with a different differential order and a nonlinear reaction term.
Therefore, adapting the framework to a different PDE requires modifying the residual factor, while the general message passing procedure remains unchanged. 

\subsection{Data subgraph}
\label{subsection:data_graph}
The data subgraph forces the neural network to learn the observed values from the initial condition
\((\vx_n,y_n)\in\mathcal D\) with a factor $f_U$. With our simple architecture, we can compute forward and backward messages only using known Gaussian properties and the approximations of the activation function moments, derived in Appendix~\ref{appendix:activation}.

\paragraph{Forward message}
Given the incoming messages to weights, we can compute the moments of $z_n:=\mW^\top\phi(\vx_n) \sim \mathcal N(\cdot;\mu_{z_n},\sigma_{z_n}^2)$ exactly. The activation moments are derived in Appendix~\ref{appendix:activation}, giving $f(\vx_n;\vtheta) \approx \mathcal N(\cdot;\mu_{h_n},\sigma_{h_n}^2)$.

\paragraph{Backward message}
To propagate the observation \(y_n\) back to the weights, we first
moment-match the nonlinear joint distribution of
\(z_n\) and \(h_n=g(z_n)\) with a Gaussian distribution. This constitutes
an additional Gaussian approximation in the backward update, while retaining
the dependence between \(z_n\) and \(h_n\) through their covariance.
Approximating the joint distribution of \(z_n\) and \(h_n\) by a Gaussian
with the corresponding moments and conditioning on the observation
\(h_n=y_n\) gives
\begin{equation*}
\begin{aligned}
    \mu_{z_n\mid y_n}
    &=
    \mu_{z_n}
    +
    \frac{\operatorname{Cov}(z_n,h_n)}
         {\sigma_{h_n}^2 + \epsilon_D^2}
    (y_n-\mu_{h_n}), \quad
    \sigma_{z_n\mid y_n}^2
    &=
    \sigma_{z_n}^2
    -
    \frac{\operatorname{Cov}(z_n,h_n)^2}
         {\sigma_{h_n}^2 + \epsilon_D^2},
\end{aligned}
\end{equation*}

where $\epsilon_D$ refers to the modeled noise in labels as shown in Figure \ref{fig:fg_single_layer} and the covariance can be evaluated using Stein's lemma, introduced in Equation~\ref{eq:stein-lemma}, $ \operatorname{Cov}(z_n,h_n) = \sigma_{z_n}^2\,\mathbb E[g'(z_n)]$.

Since
\(z_n=\sum_j\phi_j(\vx_n)W_j\), we isolate \(W_k\) and moment-match its
distribution:
\begin{equation}
\begin{aligned}
    W_k
    &=
    \frac{
        z_n-\sum_{j\neq k}\phi_j(\vx_n)W_j
    }{
        \phi_k(\vx_n)
    },\quad
    \mu_{W_k}
    &=
    \frac{
        \mu_{z_n\mid y_n}
        -\sum_{j\neq k}\phi_j(\vx_n)\mu_j
    }{
        \phi_k(\vx_n)
    },\\
    \sigma_{W_k}^2
    &=
    \frac{
        \sigma_{z_n\mid y_n}^2
        +\sum_{j\neq k}\phi_j(\vx_n)^2\sigma_j^2
    }{
        \phi_k(\vx_n)^2
    }.
\end{aligned}
\end{equation}
Thus, the backward message is approximated as  $m_{f_U\rightarrow W_k}(W_k) \approx \mathcal N(W_k;\mu_{W_k},\sigma_{W_k}^2)$.

\subsection{Message passing schedule}
\label{subsection:schedule}
Each subgraph alone can significantly narrow the solution space without satisfying the full model. As shown in Figure~\ref{fig:advection_spaces} (Appendix \ref{appendix:schedule}), the data graph learns the initial condition but not its time evolution, but already significantly narrows the solution space. The PDE graph alone just propagates the prior over time without conditioning on the observed solution. A fixed point of either subgraph satisfies only $p(\vtheta \mid\mathcal{D})$ or $p(\vtheta \mid \mathcal{C})$, rather than the joint posterior. As introduced in Section \ref{subsection:fg}, Gaussian messages combine additively in their natural parameters. We therefore interleave initial-condition and PDE updates, allowing both constraints to shape the posterior jointly and requiring empirically observed convergence to a fixed point of \(p(\vtheta\mid\mathcal{D},\mathcal{C})\) rather than either subgraph alone; refer to more details and an ablation study on this in Appendix \ref{appendix:schedule}.

\section{Evaluation}
\label{section:evaluation}
\begin{table}[t]
\centering
\small
\setlength{\tabcolsep}{4pt}
\begin{tabular}{l cccc cccc}
\toprule
 & \multicolumn{4}{c}{Advection} & \multicolumn{4}{c}{Fisher-KPP} \\
\cmidrule(lr){2-5} \cmidrule(lr){6-9}
 & Ours & VI & \textcolor[gray]{0.5}{HMC} & SGD & Ours & VI & \textcolor[gray]{0.5}{HMC} & SGD \\
\midrule
RMSE ($\times 10^{-2}$) & 4.89 & 19.5 & \textcolor[gray]{0.5}{4.85} & \textbf{4.86} & \textbf{2.72} & 72.7 & \textcolor[gray]{0.5}{2.73} & 13.7 \\
Cov.\ 90\,\% (\%) & \textbf{14.5} & 11.1 & \textcolor[gray]{0.5}{27.1} & -- & \textbf{50.0} & 0.0 & \textcolor[gray]{0.5}{68.3} & -- \\
CRPS ($\times 10^{-3}$) & \textbf{39.0} & 160 & \textcolor[gray]{0.5}{37.9} & 39.2 & \textbf{3.78} & 718 & \textcolor[gray]{0.5}{3.65} & 112 \\
\bottomrule
\end{tabular}
\caption{Predictive accuracy and uncertainty diagnostics against the exact solutions, aggregated over all runs. For the metric descriptions, refer to Appendix \ref{appendix:uncertainty-metrics}. \textbf{RMSE}: mean per-run root mean square error of the predictive mean. \textbf{Cov.\ 90\,\%}: fraction of grid points for which the 90\% predictive interval contains the exact solution. \textbf{CRPS}: continuous ranked probability score. HMC (grey) serves as the reference. Stochastic gradient descent (SGD) refers to the non-Bayesian PINN baseline (Section~\ref{subsection:solution-fidelity}); its CRPS reduces to the mean absolute error. Best results among our method, VI, and SGD are highlighted in bold.} 
\label{tab:main_results}
\end{table}
Our evaluation asks three questions: (1) whether message passing recovers accurate PDE solutions, (2) how closely its posterior marginal statistics and predictive uncertainty agree with a sampling-based HMC reference, and (3) what computational cost this requires relative to HMC and mean-field VI. We use HMC as a reference for approximate posterior inference rather than as ground truth.
We use the advection equation as a representative first-order PDE for quantitative comparison and uncertainty analysis; it has the form presented in Equation~\ref{eq:first_order_pde} with $\vb=(b,1)^\top$, and the residual function $f_R$ is defined in Equation~\ref{eq:advection_residual}. One step further, we also evaluate our approach on the Fisher-KPP equation, a semi-linear second-order PDE, whose residual is given in Equation~\ref{eq:fisher_residual}.
\paragraph{Experimental setup} We evaluate on two PDEBench~\citep{takamoto2022pdebench} datasets, which are generated by computing future states of the system from given initial conditions using known exact solutions of the PDEs. We take 1D advection across five wave speeds, $b\in\{0.1,0.4,1,2,7\}$, with ten trajectories per wave speed, and 1D reaction--diffusion across its full 16-setting parameter grid, $(\nu,\rho)\in\{0.5,1,2,5\}\times\{1,2,5,10\}$, with five trajectories per setting, all drawn from canonical seeded subsets. All the baselines we compare to share the same architecture, feature basis, labeled and collocation points, initialization, and random seeds as our model.  All experiments, including the timing measurements, ran on a single CPU-only machine; more details on the experimental setup are provided in Appendix~\ref{appendix:experimental_details}.

\subsection{Solution fidelity}
\label{subsection:solution-fidelity}

\paragraph{Message passing compared to exact solutions}
We first compare our predictions against exact solutions, i.e., those generated by PDEBench datasets. Figure~\ref{fig:advection_pdebench} shows the predictions for one of the advection trajectories in detail: across the evaluated wave speeds, the posterior mean follows the correct pattern of the ground-truth solution. At the same time, the posterior variance exhibits a structured dependence on the prediction horizon: it is lowest near the observed initial condition and generally increases as the solution is propagated further into the future. Figure~\ref{fig:uq}(c) shows this pattern across the evaluated advection speeds, and related behavior is observed for Fisher-KPP in Figures~\ref{fig:fisher-pdebench} and \ref{fig:fisher-simple}. We interpret this as evidence that the inferred uncertainty responds to the available information and PDE dynamics.

\paragraph{Message passing compared to a baseline}
As a point-estimate reference, we adopt a non-Bayesian PINN baseline from \cite{takamoto2022pdebench} and train it with Adam using the standard PINN loss (SGD baseline).
For all investigated trajectories, we observe that our approach achieves performance very close to that of the baseline model: Figure~\ref{fig:mp_vs_pdebench_scatter_grid} for the advection and Figure \ref{fig:mp_vs_pdebench_scatter_grid_fisher} for the Fisher-KPP show that the mean squared error (MSE) and the PDE residual $R$ are almost identical, while for the Fisher-KPP, our approach is even slightly better.
In other words, this demonstrates that, given a nearly identical setup, our message passing training procedure yields results very close to those of a classical gradient-based training procedure.
Table~\ref{tab:main_results} summarizes the predictive results: our root mean square error (RMSE) is close to HMC and the baseline referred to as stochastic gradient descent (SGD) on advection and to HMC on Fisher-KPP, although 90\,\% coverage remains below nominal.

\begin{figure}[!htb]
    \centering
    \includegraphics[
    width=1.\linewidth
]{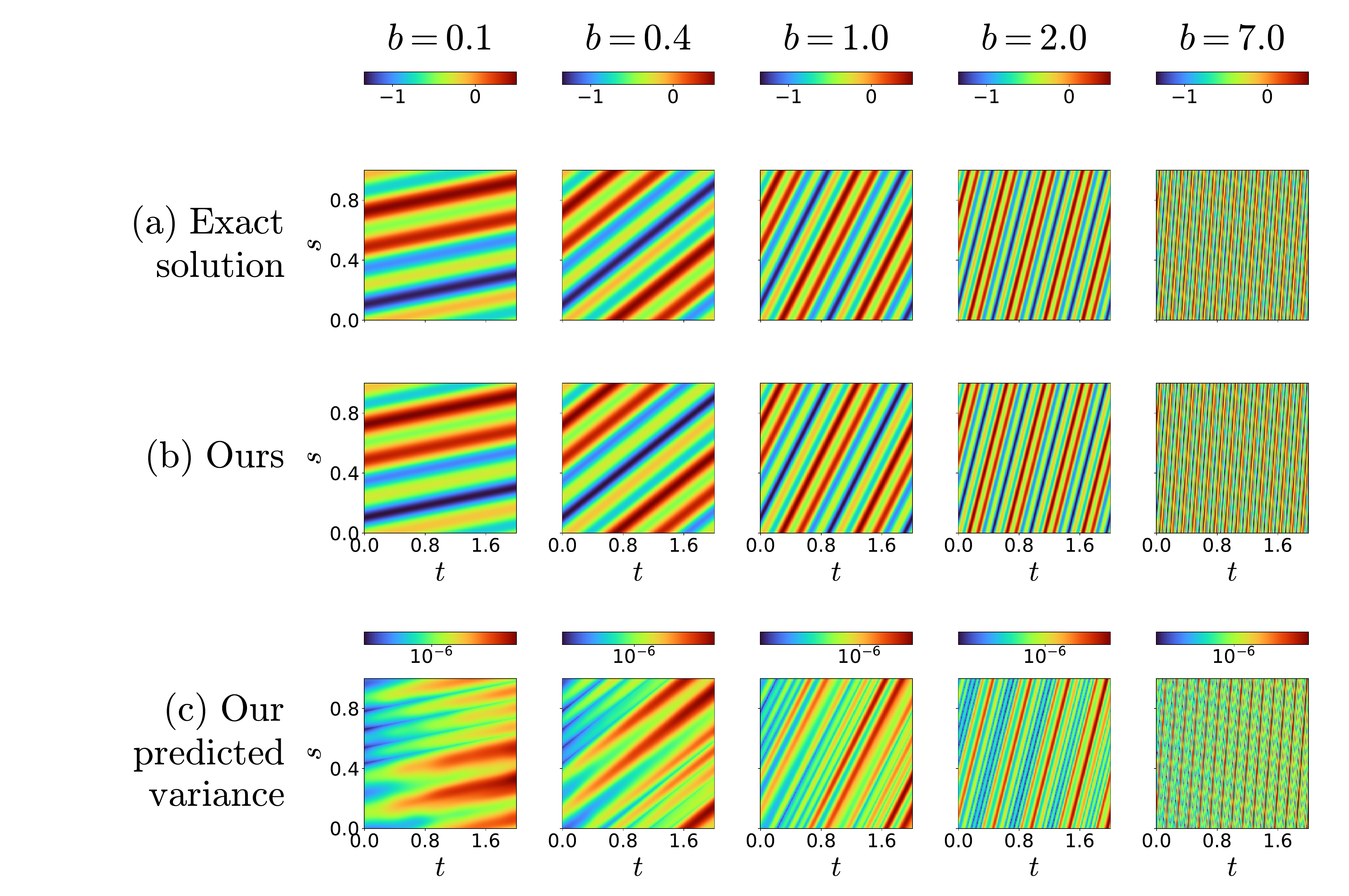}
    \caption{Comparison of the message passing approach (ours) and the baseline solution from PDEBench for the advection equation across different wave speeds, $b\in\{0.1, 0.4, 1, 2, 7\}$. All heatmaps are plotted in the coordinates of $\vx=(s, t)$, where $s$ denotes space and $t$ denotes time. \textbf{(a)} the exact solution; \textbf{(b)} predicted mean values of our message passing solution, which preserve the propagation pattern across the evaluated wave speeds; \textbf{(c)} predicted variance values of our message passing solution, showing increasing uncertainty further in time from the observed initial condition.}
    \label{fig:advection_pdebench}
\end{figure}

\subsection{Bayesian comparison}
\label{subsection:bayesian_comparison}
We compare our approach to HMC and VI in terms of the similarity of the inferred parameter posteriors and the training and inference time costs. 
To compare the predicted uncertainty, we use the continuous ranked probability score (CRPS) and 90\,\% coverage (see Appendix \ref{appendix:uncertainty-metrics}).
Table~\ref{tab:main_results} shows CRPS close to HMC for both PDEs, with lower 90\,\% coverage, while VI is markedly worse.
\paragraph{Message passing compared to Hamiltonian Monte Carlo}
For the HMC baseline, we use Algorithm 1 from \cite{yang2021b}, with a Metropolis-Hastings acceptance step and Gaussian likelihoods on the observed initial condition and the PDE residual. To ensure a fair comparison, we reimplemented HMC in Julia using the same architecture, priors, and random seeds as our message passing approach.
Using HMC as a sampling-based reference, we find close agreement in posterior marginal means, as shown in Figures~\ref{fig:bayesian_comparison}(a) and \ref{fig:bayesian_comparison_fisher}(a). The marginal standard deviations do not exactly match, but generally follow the same structure and are substantially closer to the HMC reference than those obtained with the mean-field VI baseline; Figures~\ref{fig:uq} and \ref{fig:uq_fisher} quantify the remaining differences in predictive uncertainty. These comparisons concern marginal first- and second-order statistics and do not establish recovery of the full joint HMC posterior.
In the evaluated models, message passing is substantially cheaper computationally than HMC because it avoids posterior sampling (Figures~\ref{fig:bayesian_comparison_time} and \ref{fig:bayesian_comparison_time_fisher}).
\paragraph{Message passing compared to variational inference} 
For the VI baseline, we use Algorithm 2 from \cite{yang2021b}; again, we reimplement it in Julia to enable a fair comparison. Figure~\ref{fig:bayesian_comparison_ours_vi} for the advection and Figure \ref{fig:bayesian_comparison_ours_vi_fisher} for the Fisher-KPP show that our approach still infers similar posterior means, but for posterior variances, VI yields different estimates from both HMC and our approach -- Figures ~\ref{fig:bayesian_comparison_vi_hmc} and \ref{fig:bayesian_comparison_vi_hmc_fisher} additionally illustrate the difference between HMC and VI for both equations, respectively. In general, VI is known to often underestimate posterior variances \citep{blei2017variational}; therefore, we regard the comparison with HMC as the more reliable reference. The computational efficiency of our approach compared to VI is shown in Figure~\ref{fig:bayesian_comparison_time} for the advection and in Figure \ref{fig:bayesian_comparison_time_fisher} for the Fisher-KPP: our training is faster at every wave speed, while inference is equally fast, since both methods compute the predictive distribution in closed form.
\begin{figure}
    \centering
    \includegraphics[width=.9\linewidth]{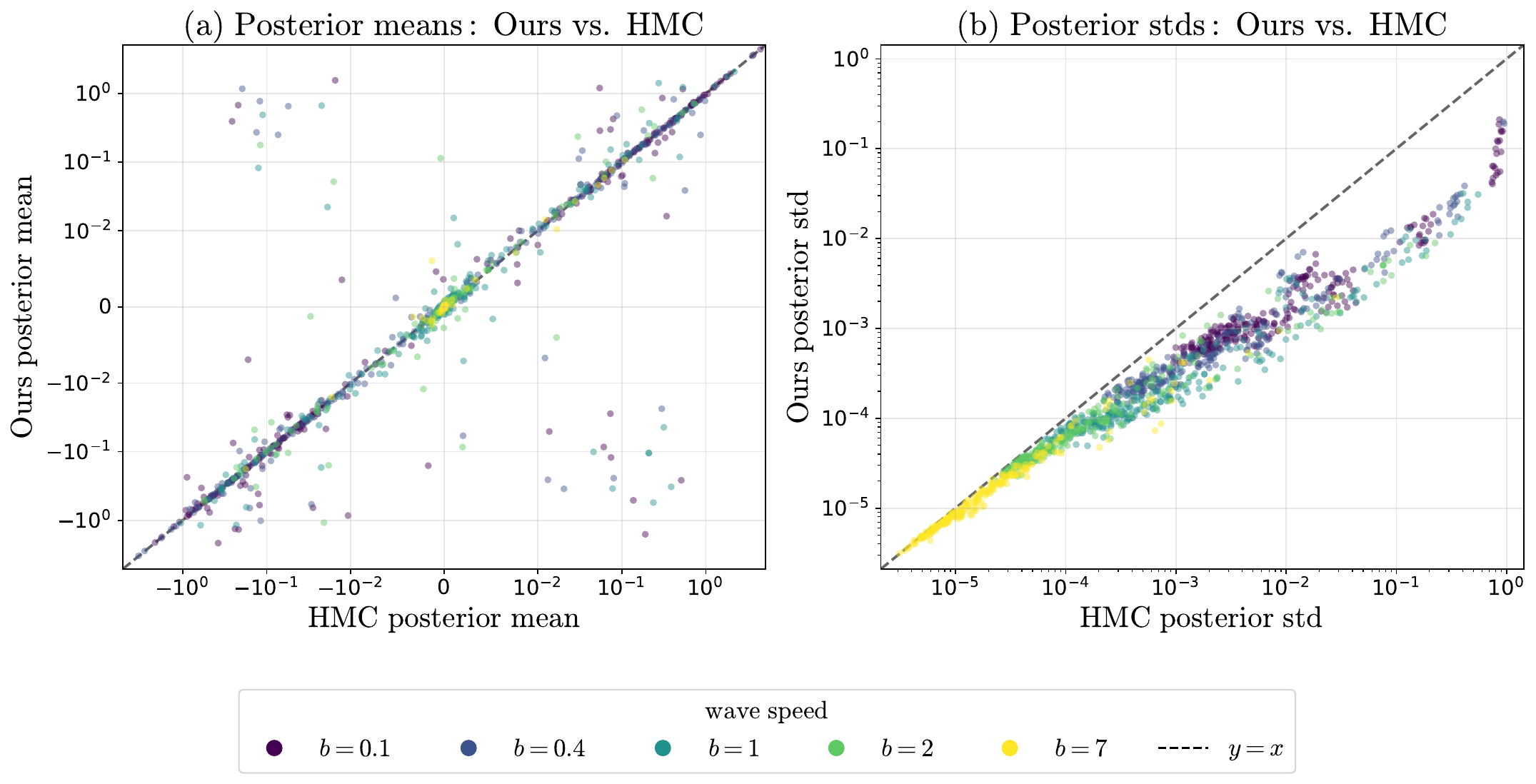}
    \caption{Comparison of our approach (ours) to the HMC implemented based on \cite{yang2021b} for the advection equation. We map posterior weights of HMC (x-axis) to our approach (y-axis) - points in \textbf{(a)} refer to the mean values, points in \textbf{(b)} refer to the standard deviations; our approximate approach converges to posterior weights close to the ones from HMC. For readability, 25 random weights per trajectory are shown.}
    \label{fig:bayesian_comparison}
\end{figure}
\begin{figure}
    \centering
    \includegraphics[width=1.\linewidth]{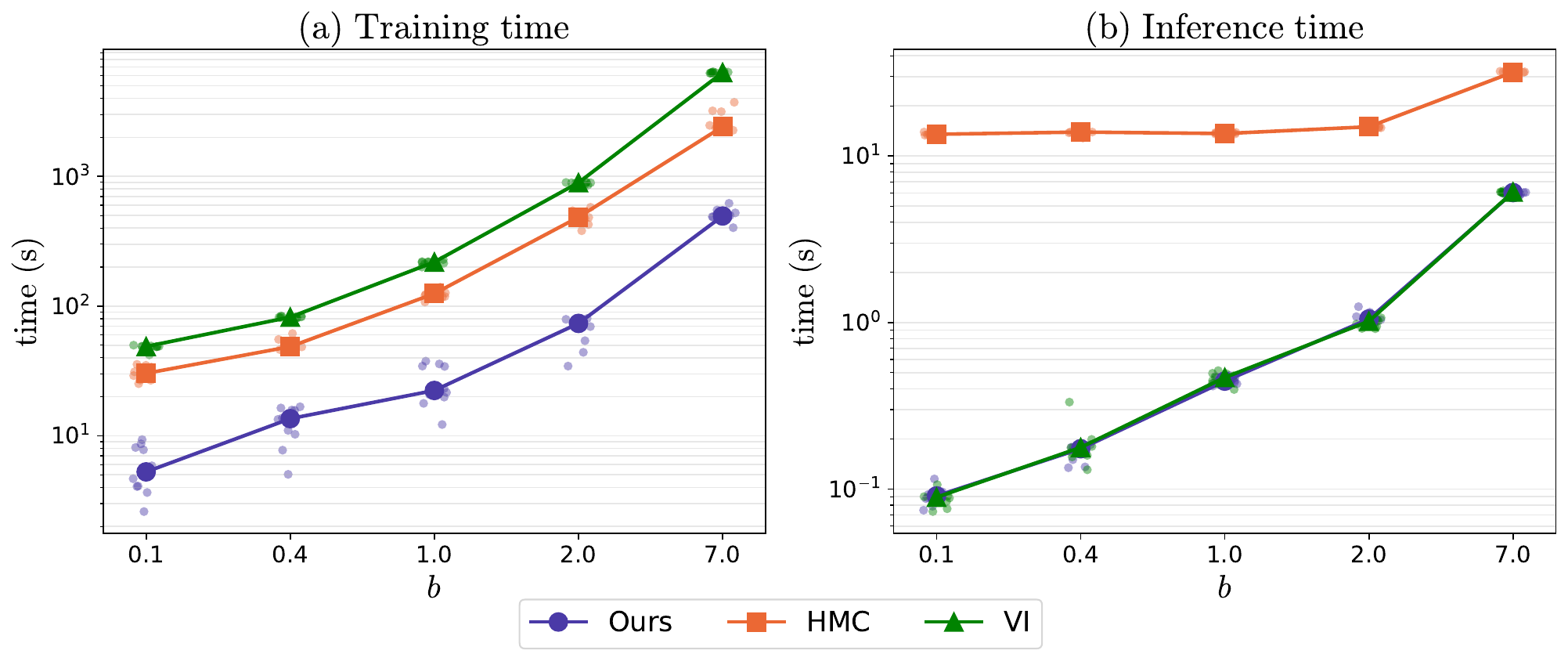}
    \caption{Computational cost of our approach (ours), HMC and VI for the advection equation across wave speeds $b$, ten trajectories each (points: trajectories; lines: medians). \textbf{(a)} Training time, cold start; \textbf{(b)} inference time of the predictive mean and variance on the full grid, steady state. Our training is up to $6.6\times$ faster than HMC and up to $13\times$ faster than VI; our closed-form inference is as fast as VI's and up to $150\times$ faster than HMC's sampling-based inference.}
    \label{fig:bayesian_comparison_time}
\end{figure}
\section{Related work}
\label{section:related-work}

\paragraph{Sampling and variational inference}
The foundational Bayesian approach, Bayesian PINN (B-PINN; \citealt{yang2021b}), models neural network parameters as random variables and infers their posterior using either Hamiltonian Monte Carlo (HMC) or variational inference (VI). HMC is shown to provide more accurate posterior estimates; however, it requires computationally expensive posterior sampling. Moreover, VI learns an approximate posterior using stochastic gradient-based optimization; this is where our approach differs: we propose a more natural
probabilistic solution based on message passing, in which inference is performed by directly propagating distributions through the model. Another line of work aims to improve HMC performance through a randomized loss function \citep{zong2025randomized}, but still relies on extensive posterior sampling. \citep{raj2025deep} proposes a more advanced VI approach, where a neural operator parameterizes the variational posterior and the corresponding evidence lower bound (ELBO)
incorporates both data and PDE residual terms, but inference still relies on gradient-based optimization and Monte Carlo sampling.\\
\paragraph{Physics-informed priors}
In contrast to approaches that incorporate the PDE directly during posterior inference, the work of \citep{menicali2025bayesian} uses physics to construct the Bayesian neural network (BNN) prior. The prior means are first obtained by optimizing a PINN objective that includes the PDE, initial conditions, and boundary conditions. The resulting physics-informed prior is then updated with the observed data via Bayes by Backprop \citep{blundell2015weight}, which performs variational inference via stochastic gradient-based optimization.\\
\paragraph{Gaussian processes}
A different line of work replaces the BNN with a Gaussian process (GP), placing a probability distribution directly over the PDE solution rather than over neural network parameters. Since derivatives of a GP are again jointly Gaussian under suitable kernels, the solution and its derivatives can be modeled jointly, which provides a natural way to incorporate differential constraints. \citep{paun2025physics} constructs such a physics-informed GP by conditioning on the PDE at a finite set of points and performs posterior inference using HMC.
\citep{chkrebtii2016bayesian} jointly models the solution and its derivatives with a GP and sequentially updates their distributions using local
differential-equation evaluations obtained from sampled states.
\section{Conclusion}

\paragraph{Summary}
Factor graphs allow different sources of prior knowledge to be incorporated directly into the probabilistic model through separate constraints.
Building on this property, this work proposes a novel perspective on solving PDEs without sampling or gradient-descent optimization: formulating them as factor graphs for Bayesian neural networks that encode prior physical knowledge. Inference is performed via message passing, and because both subgraphs contain deterministic relations represented by Dirac delta factors, DMA is a suitable message passing framework for our factor graph.
Our evaluation of the advection and Fisher-KPP equations demonstrates the feasibility of our approach, and the Bayesian comparison shows that it recovers posterior marginal means close to the HMC reference at substantially lower computational cost, while producing uncertainty estimates that are closer in structure to HMC than the evaluated mean-field VI baseline.

\paragraph{Discussion}
This work provides a first demonstration of Bayesian PDE learning through message passing on a factor graph. We deliberately focus on a shallow Bayesian neural network with a single nonlinear output unit and two representative PDEs. Within this controlled setting, we show that the same inference framework can recover PDE solutions and structured predictive uncertainty across equations with different differential and nonlinear structures. Our goal at this stage is to establish the feasibility of this new inference perspective rather than its full scalability. Extending the framework to higher-dimensional PDEs and deeper neural networks is therefore a promising direction for future work.
\subsection*{AI use statement}

In this work, we used generative AI tools for editorial assistance and limited programming assistance.
We have not used generative AI tools for generating research ideas, developing the proposed methodology, deriving results, conducting the literature review, interpreting experimental findings, or drawing conclusions.
We have reviewed all AI-assisted work by reviewing the AI-assisted code implementations. 
We take responsibility for the final content of this work,
including text, claims or artifacts produced with the aid of generative AI.

\subsection*{Ethics statement}

This work does not raise any specific ethical concerns.

\subsection*{Reproducibility statement}

We provide all experimental details in Appendix~\ref{appendix:experimental_details}, including the hardware configuration, hyperparameters, and other implementation details. The reference implementation is available in the anonymous repository:
\url{https://anonymous.4open.science/r/mp-pinn-release-5CFF/README.md}.
It includes scripts for generating the experimental data and reproducing the paper's results.



\newpage

\bibliography{main_bibliography}
\bibliographystyle{iclr2027_conference}

\newpage
\appendix
\section{Appendix}

\subsection{Details on the message passing algorithm}
\label{appendix:message-passing}
This appendix provides the details needed to reproduce the inference procedure introduced in Section~\ref{subsection:fg}. Appendix~\ref{appendix:sum-product} summarizes the sum-product algorithm~\citep{kschischang2001factor} and the Gaussian natural parameter arithmetic used by our implementation. Appendix~\ref{appendix:schedule} then specifies the randomized update schedule used on the PDE factor graph from Section~\ref{section:approach}.

\subsubsection{Sum-product algorithm}
\label{appendix:sum-product}

The sum-product algorithm performs inference by exchanging local messages between variable and factor nodes on a factor graph. Intuitively, a variable-to-factor message summarizes the information about a variable provided by all its neighboring factors except the receiving factor, while a factor-to-variable message combines the factor with the incoming information from all other variables connected to it. The marginal of a variable is then computed by combining all messages incoming to that variable. More specifically, given a set of random variables $ \mathbf{X}=(X_1,\ldots,X_K)$ and a set of factors $\{f_i\}_{i=1}^N$, marginals and messages follow the following rules:
\begin{subequations}
\label{eq:sum-product}
\begin{align}
p(\boldsymbol\xi)
&\propto
\prod_{i=1}^{N}
f_i\!\left(\boldsymbol\xi_{\operatorname{ne}(f_i)}\right),
\\
m_{X_k\rightarrow f_i}(x_k)
&\propto
\prod_{f_j\in\operatorname{ne}(X_k)\setminus\{f_i\}}
m_{f_j\rightarrow X_k}(x_k),
\\
m_{f_i\rightarrow X_k}(x_k)
&\propto
\int
f_i\!\left(\boldsymbol\xi_{\operatorname{ne}(f_i)}\right)
\prod_{X_j\in\operatorname{ne}(f_i)\setminus\{X_k\}}
m_{X_j\rightarrow f_i}(x_j)
\,d\boldsymbol\xi_{\operatorname{ne}(f_i)\setminus\{X_k\}},
\end{align}
\end{subequations}

where $m_{X_k\rightarrow f_i}(x_k)$ is a message from a variable node $X_k$ to a factor node $f_i$ and $m_{f_i\rightarrow X_k}(x_k)$ is a message from a factor node $f_i$ to a variable node $X_k$.

In our implementation, all Gaussian messages and marginals are represented in terms of their natural parameters. For a Gaussian density with mean $\mu$ and variance $\sigma^2$, we define the precision as $\kappa := \frac{1}{\sigma^2}$ and the precision-weighted mean as $\tau:= \frac{\mu}{\sigma^2}$, and denote its natural parameterization by $\mathcal G(\cdot;\tau,\kappa)$. Consider two Gaussian densities over the same variable $x$,
\[
p(x)
=
\mathcal G(x;\tau_p,\kappa_p),
\qquad
q(x)
=
\mathcal G(x;\tau_q,\kappa_q).
\]
Then, the multiplication of Gaussian densities corresponds to the addition of their natural parameters,
\begin{equation*}
p(x)q(x)
\propto
\mathcal G(x;\tau_p+\tau_q,\kappa_p+\kappa_q),
\end{equation*}
while division corresponds to subtraction,
\begin{equation*}
\frac{p(x)}
     {q(x)}
\propto
\mathcal G(x;\tau_p-\tau_q,\kappa_p-\kappa_q),
\end{equation*}
provided that the resulting precision is positive.

We use these properties for each message update in the employed sum-product algorithm. Besides the convenience, these operations of summation and subtraction are also less computationally expensive than multiplication or division that classic parametrization would require.

\subsubsection{Randomized message passing schedule}
\label{appendix:schedule}

Algorithm~\ref{alg:message_passing} summarizes the randomized message passing
schedule for the PDE factor graph defined in Section~\ref{subsection:pde_graph}. The updates follow the sum-product rules introduced in Section~\ref{subsection:fg} and summarized in Equation~\ref{eq:sum-product}. At the beginning of each epoch, all data and PDE residual factors are placed in a randomized order. Each factor update consists of three stages: (1) a forward update that propagates the incoming weight messages through the factor to obtain an approximate Gaussian message for the network output \(F\) or PDE residual \(R\); (2) an observation update that combines this message with the data observation or zero-residual constraint; and (3) a backward update that propagates the resulting information to the weights and replaces the factor's previously cached messages.

The randomized schedule is motivated in Section~\ref{subsection:schedule}. The belief of each weight \(W_k\) combines the prior, data, and PDE messages as
\begin{equation*}
\begin{aligned}
q(W_k)
&\propto
q^{(0)}(W_k)\,
q^{\mathcal D}(W_k)\,
q^{\mathcal C}(W_k)
\\
&\propto
\exp\left\{
-\frac{1}{2}
\left(
\kappa_k^{(0)}
+\kappa_k^{\mathcal D}
+\kappa_k^{\mathcal C}
\right)W_k^2
+
\left(
\tau_k^{(0)}
+\tau_k^{\mathcal D}
+\tau_k^{\mathcal C}
\right)W_k
\right\},
\end{aligned}
\end{equation*}
where $(\kappa_k^{(0)},\tau_k^{(0)})$, 
$(\kappa_k^{\mathcal D},\tau_k^{\mathcal D})$, and 
$(\kappa_k^{\mathcal C},\tau_k^{\mathcal C})$ denote the natural parameters 
contributed by the prior, data subgraph, and PDE subgraph, respectively. Updating either subgraph separately may lead to a fixed point that satisfies only the data constraints or only the PDE constraints, while substantially narrowing the weight beliefs before the other subgraph is considered. Interleaving individual data and PDE-residual factors allows both sources of information to influence the
beliefs throughout inference. Our ablation study illustrates this behavior in Figure~\ref{fig:advection_spaces}:
converging the data subgraph separately produces overly concentrated intermediate beliefs, whereas the interleaved schedule empirically reaches a fixed point that reflects both the observed data and the PDE constraints. This observation is empirical and does not constitute a general convergence
guarantee for loopy message passing.

\begin{algorithm}[t]
\caption{Randomized message passing schedule on the PDE factor graph}
\label{alg:message_passing}
\begin{algorithmic}[1]

\Require Factors
$\mathcal F=\mathcal F_{\mathcal D}\cup\mathcal F_{\mathcal C}$,
prior beliefs $q^{(0)}(W_k)$, maximum number of epochs
$E_{\max}$ (default: $300$), and convergence tolerance
$\delta$ (default: $10^{-5}$)

\State Initialize every cached factor-to-weight message with zero natural
parameters:
\Statex \hspace{\algorithmicindent}
$\bigl(\kappa_{i\rightarrow k},\tau_{i\rightarrow k}\bigr)
\gets(0,0)$
for every $f_i\in\mathcal F$ and
$W_k\in\operatorname{ne}(f_i)$

\State Initialize the weight beliefs:
$q(W_k)\gets q^{(0)}(W_k)$ for every $W_k$

\For{$e=1,\ldots,E_{\max}$}
    \State Draw a fresh random permutation $\pi_e$ of $\mathcal F$
    \State $\Delta_e\gets0$

    \For{$f_i$ in the order $\pi_e$}

        \For{$W_k\in\operatorname{ne}(f_i)$}
            \State Extract the factor-excluded incoming message as given in Equation \ref{eq:sum-product}:
            \Statex \hspace{\algorithmicindent}
            $\displaystyle
            m_{W_k\rightarrow f_i}(W_k)
            \propto
            \frac{q(W_k)}
                 {m_{f_i\rightarrow W_k}(W_k)}$
        \EndFor

        \State Set the forward inputs to
        $\bar m_{ik}(W_k)\gets m_{W_k\rightarrow f_i}(W_k)$
        for every $W_k\in\operatorname{ne}(f_i)$

        \State \textbf{(1) Forward update:} propagate
        $\left\{
        \bar m_{ik}
        \right\}_{W_k\in\operatorname{ne}(f_i)}$
        through $f_i$ and moment-match the resulting output message
        $m_{f_i\rightarrow F_i}$ if
        $f_i\in\mathcal F_{\mathcal D}$, or
        $m_{f_i\rightarrow R_i}$ if
        $f_i\in\mathcal F_{\mathcal C}$

        \State \textbf{(2) Observation update:} combine the forward
        message with the corresponding Gaussian observation -- observed data point $y_i$ for a data factor of zero-mean residual for a PDE residual factor:
        \Statex \hspace{\algorithmicindent}
        $\displaystyle
        q(F_i)\propto
        m_{f_i\rightarrow F_i}(F_i)\,
        \mathcal N
        \!\left(F_i;y_i,\epsilon_D^{2}\right),
        \qquad f_i\in\mathcal F_{\mathcal D},$
        \Statex \hspace{\algorithmicindent}
        $\displaystyle
        q(R_i)\propto
        m_{f_i\rightarrow R_i}(R_i)\,
        \mathcal N
        \!\left(R_i;0,\epsilon_C^{2}\right),
        \qquad f_i\in\mathcal F_{\mathcal C}.$

        \State \textbf{(3) Backward update:} use the updated output
        belief $q(F_i)$ or $q(R_i)$ to compute the Gaussian messages
        $\widetilde m_{f_i\rightarrow W_k}$ for every
        $W_k\in\operatorname{ne}(f_i)$

        \For{$W_k\in\operatorname{ne}(f_i)$}
            \State $q_{\mathrm{old}}(W_k)\gets q(W_k)$

            \State Replace the cached message and update the weight belief:
            \Statex \hspace{\algorithmicindent}
            $\displaystyle
            q(W_k)\gets
            m_{W_k\rightarrow f_i}(W_k)\,
            \widetilde m_{f_i\rightarrow W_k}(W_k)$

            \State
            $m_{f_i\rightarrow W_k}
            \gets
            \widetilde m_{f_i\rightarrow W_k}$

            \State Update the maximum belief change:
            \Statex \hspace{\algorithmicindent}
            $\displaystyle
            \Delta_e\gets
            \max\!\left\{
                \Delta_e,\,
                D_{\mathrm{KL}}\!\left(
                    q_{\mathrm{old}}(W_k)
                    \,\|\,q(W_k)
                \right)
            \right\}$
        \EndFor
    \EndFor

    \If{$\Delta_e<\delta$}
        \State \textbf{break}
    \EndIf
\EndFor

\State \Return $\left\{q(W_k)\right\}_k$
\end{algorithmic}
\end{algorithm}

\begin{figure}
    \centering
    \includegraphics[width=0.7\linewidth]{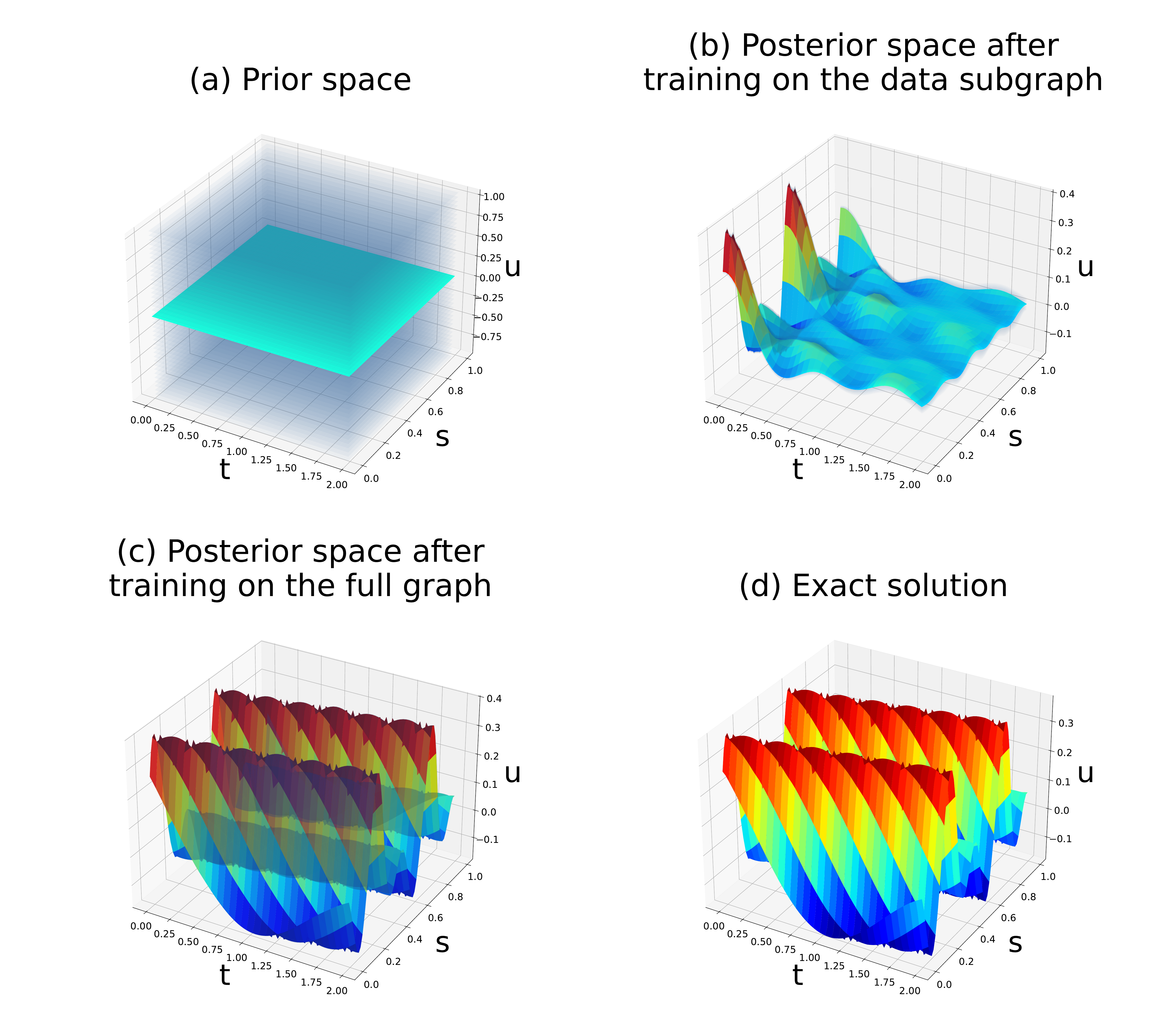}
    \caption{Solution spaces for the advection equation ($b=0.1$) for prior and posterior parameters shown as predictive mean (colorful) $\pm\sigma$ (gray shade). \textbf{(a)} prior space before any training; \textbf{(b)} posterior space with parameters trained only on initial condition - the function fits perfectly for $t=0$, but does not learn any time dependencies and already narrows the space too much; \textbf{(c)} posterior space with parameters trained with a mixed schedule of initial condition data passes and PDE passes - the function learns both spatial and time patterns correctly; \textbf{(d)} exact solution of the PDE.}
    \label{fig:advection_spaces}
\end{figure}

\subsection{Theoretical properties}
\label{appendix:theory}

Our target distribution is the Bayesian posterior in Equation~\ref{eq:bayesian_posterior}.
For a nonlinear neural network, the messages induced by the data and PDE
factors are generally non-Gaussian; therefore, posteriors computed by our algorithm are not exact. Instead, each message is moment-matched to a Gaussian, producing \textit{an approximate belief}
\(q(\vtheta)\). After this convergence, these beliefs are not guaranteed to match the exact posterior, i.e., $q^\star(\vtheta) \not\equiv
p(\vtheta\mid\mathcal D,\mathcal C)$ - the equality holds only if all factors produce only Gaussians, and messages therefore can be computed exactly.
We therefore evaluate the resulting predictive mean and uncertainty empirically; refer to Section~\ref{subsection:bayesian_comparison} for this evaluation. We also empirically observe that interleaving data and PDE updates is important for obtaining solutions that satisfy both sources of information; this behavior is not claimed as a convergence guarantee.

\subsection{Differentiable and decorrelated basis}
\label{appendix:basis}

Inspired by Fourier series, for each coordinate \(x_i\), \(i=1,\ldots,d\),
we define
\begin{equation*}
\psi_i(x_i)
=
\Bigl[
1,\;
\bigl(
\sin(2\pi kx_i),
\cos(2\pi kx_i)
\bigr)_{k=1}^{K_i}
\Bigr],
\end{equation*}
where \(K_i\) determines the number of frequencies. We combine these coordinate-wise bases using the Kronecker product, $\phi(\vx) = \bigotimes_{i=1}^{d}\psi_i(x_i)$, which produces all combinations of the coordinate-wise basis functions.

In the factor graph, the basis enters inference only through the factors corresponding to the data observations and PDE constraints. If different basis components behave similarly at the data and PDE points, their corresponding weights are not individually well-identified, and hence the posterior exhibits strong correlations among these weights. These correlations could be represented explicitly by maintaining a full multivariate Gaussian distribution over the weights; however, this would substantially increase the algorithm's cost. Instead, inspired by the principal component analysis (PCA) decorrelation transform \citep{kessy2018optimal}, we \textit{rotate the basis} so that components with similar effects are replaced by new components that more distinctly affect the constraints, reducing correlations between their corresponding weights.
We evaluate the basis components entering the data and PDE constraints as
for advection, with $\vr_{\mathrm{adv}}(\vx):=\sum_{i=1}^{d} b_i\,\partial_{x_i}\phi(\vx)$,
\begin{equation*}
\mA_{\mathcal D}
=
\begin{bmatrix}
\phi(\vx_1)^\top\\
\vdots\\
\phi(\vx_N)^\top
\end{bmatrix},
\qquad
\mA_{\mathcal C}
=
\begin{bmatrix}
\vr_{\mathrm{adv}}(\vx_1)^\top\\
\vdots\\
\vr_{\mathrm{adv}}(\vx_M)^\top
\end{bmatrix}.
\end{equation*}

We aim to evaluate the similarities of the basis components. We construct a weighted Gram matrix that contains non-normalized cosine similarities between the components:
\begin{equation*}
\mG
=
\frac{1}{\epsilon_D^2}
\mA_{\mathcal D}^{\top}\mA_{\mathcal D}
+
\frac{1}{\epsilon_C^2}
\mA_{\mathcal C}^{\top}\mA_{\mathcal C}.
\end{equation*}

We compute the eigenvalue decomposition of this symmetric matrix as  $\mG=\mV\mLambda\mV^\top$
The eigenvectors define a rotation of the original basis, for which $\mV^\top\mG\mV=\mLambda$.
Since \(\mLambda\) is diagonal, the cosine similarity between any two distinct rotated components is zero.
Let \(\widetilde{\mV}\) contain the selected eigenvectors, then we
redefine the basis as $\widetilde{\phi}(\vx) = \widetilde{\mV}^{\top}\phi(\vx)$ and, correspondingly, $\widetilde{\vr}_{\mathrm{adv}}(\vx) = \widetilde{\mV}^{\top}\vr_{\mathrm{adv}}(\vx)$.
The corresponding modification for the nonlinear Fisher-KPP residual is given in Appendix~\ref{appendix:basis_fisher}.
\subsubsection{Adaptation to the Fisher-KPP equation}
\label{appendix:basis_fisher}

For the Fisher-KPP equation (Appendix~\ref{appendix:kpp}), the data matrix
$\mA_{\mathcal D}$ is unchanged. Its residual, however, is nonlinear in the weights: the
reaction term $-\rho f(1-f)$ and the quadratic term arising from the chain rule
(Appendix~\ref{appendix:kpp}) define no weight-independent basis components. We therefore
build $\mA_{\mathcal C}$ from the linear part of the operator,
\begin{equation*}
\vr(\vx)
=
\frac{\partial\phi(\vx)}{\partial t}
-
\nu\,\frac{\partial^{2}\phi(\vx)}{\partial s^{2}},
\end{equation*}
which dominates the correlations the rotation is meant to remove: the diffusion term acts on
spatial frequency $k$ as $\nu(2\pi k)^{2}$, while the reaction term is bounded. These
linearized rows enter only $\mG$, i.e., only the choice of basis; the residual function $f_R$ itself
uses the full nonlinear expression of Equation~\ref{eq:fisher_residual} in the rotated basis
$\widetilde{\phi}$.


%
\subsection{Derivations for the activation function}
\label{appendix:activation}

This section derives the forward and backward messages for the activation function $g(\cdot)$ in Equation~\ref{eq:activation} and for its derivative. Because the activation is piecewise exponential, the required moments reduce to conditional log-normal moments. Lemma~\ref{lemma:conditional_lognormal} gives these moments. Appendix~\ref{appendix:activation-messages} applies them to $g$, while Appendix~\ref{appendix:activation-derivative-messages} treats $g'$.

\begin{Lemma}
\label{lemma:conditional_lognormal}
Conditional moments of the lognormal distribution can be computed exactly.
Let \(X\sim\mathcal{N}(\cdot;\mu_X,\sigma_X^2)\). Then, for any
\(\lambda\in\mathbb{R}\),
\begin{equation*}
\begin{aligned}
\mathbb{E}\!\left[
e^{\lambda X}\mathbf{1}_{\{X>0\}}
\right]
&=
e^{\lambda\mu_X+\frac12\lambda^2\sigma_X^2}
\Phi\!\left(
\frac{\mu_X+\lambda\sigma_X^2}{\sigma_X}
\right),
\\
\mathbb{E}\!\left[
e^{\lambda X}\mathbf{1}_{\{X\le0\}}
\right]
&=
e^{\lambda\mu_X+\frac12\lambda^2\sigma_X^2}
\Phi\!\left(
-\frac{\mu_X+\lambda\sigma_X^2}{\sigma_X}
\right).
\end{aligned}
\end{equation*}
\end{Lemma}

\begin{proof}

By definition, the expectation of a function of \(X\) is the integral of that function with respect to the density of \(X\). Since the indicator functions restrict \(X\) to \(X>0\) or \(X\leq 0\), respectively, the corresponding integration domains are \((0,\infty)\) and \((-\infty,0]\). We can therefore rewrite the required expressions as

\begin{align*}
\mathbb{E}\!\left[
e^{\lambda X}\mathbf{1}_{\{X>0\}}
\right]
&=
\frac{1}{\sigma_X\sqrt{2\pi}}
\int_0^\infty
\exp\!\left(
\lambda x-\frac{(x-\mu_X)^2}{2\sigma_X^2}
\right)dx
\\
&=
e^{\lambda\mu_X+\frac{1}{2}\lambda^2\sigma_X^2}
\frac{1}{\sigma_X\sqrt{2\pi}}
\int_0^\infty
\exp\!\left(
-\frac{(x-\mu_X-\lambda\sigma_X^2)^2}
{2\sigma_X^2}
\right)dx
\\
&=
e^{\lambda\mu_X+\frac{1}{2}\lambda^2\sigma_X^2}
\Phi\!\left(
\frac{\mu_X+\lambda\sigma_X^2}{\sigma_X}
\right),
\\[1ex]
\mathbb{E}\!\left[
e^{\lambda X}\mathbf{1}_{\{X\le0\}}
\right]
&=
\frac{1}{\sigma_X\sqrt{2\pi}}
\int_{-\infty}^0
\exp\!\left(
\lambda x-\frac{(x-\mu_X)^2}{2\sigma_X^2}
\right)dx
\\
&=
e^{\lambda\mu_X+\frac{1}{2}\lambda^2\sigma_X^2}
\frac{1}{\sigma_X\sqrt{2\pi}}
\int_{-\infty}^0
\exp\!\left(
-\frac{(x-\mu_X-\lambda\sigma_X^2)^2}
{2\sigma_X^2}
\right)dx
\\
&=
e^{\lambda\mu_X+\frac{1}{2}\lambda^2\sigma_X^2}
\Phi\!\left(
-\frac{\mu_X+\lambda\sigma_X^2}{\sigma_X}
\right).
\end{align*}
\end{proof}

\subsubsection{Activation function messages}
\label{appendix:activation-messages}

\paragraph{Forward message}
For the defined $g(x)$, the forward moments follow directly from Lemma~\ref{lemma:conditional_lognormal}:

\begin{equation*}
\begin{aligned}
\mathbb{E}[g(X)]
&=
\mathbb{E}\!\left[e^{\beta X}\mathbf{1}_{\{X>0\}}\right]
-
\mathbb{P}(X>0)
-
\alpha
\left(
\mathbb{E}\!\left[e^{-\beta X}\mathbf{1}_{\{X\le0\}}\right]
-
\mathbb{P}(X\le0)
\right),
\\
\mathbb{E}[g(X)^2]
&=
\mathbb{E}\!\left[e^{2\beta X}\mathbf{1}_{\{X>0\}}\right]
-
2\mathbb{E}\!\left[e^{\beta X}\mathbf{1}_{\{X>0\}}\right]
+
\mathbb{P}(X>0)
\\
&\quad+
\alpha^2
\left(
\mathbb{E}\!\left[e^{-2\beta X}\mathbf{1}_{\{X\le0\}}\right]
-
2\mathbb{E}\!\left[e^{-\beta X}\mathbf{1}_{\{X\le0\}}\right]
+
\mathbb{P}(X\le0)
\right).
\end{aligned}
\end{equation*}

\paragraph{Backward message}

Given $\alpha>0$ and $\beta>0$ in Equation~\ref{eq:activation}, the sign of $Y=g(X)$ uniquely identifies the corresponding branch: $Y>0$ implies $X>0$, while $Y\leq0$ implies $X\leq0$. Since each branch is strictly monotonic, $g^{-1}(Y)$ has a closed form:

\begin{equation}
\label{eq:activation-inverse}
g^{-1}(Y)
=
\begin{cases}
\dfrac{1}{\beta}\log(1+Y), & Y>0,\\[2mm]
-\dfrac{1}{\beta}\log\!\left(1-\dfrac{Y}{\alpha}\right), & Y\leq0.
\end{cases}
\end{equation}

For any positive random variable $Z$ with mean $\mu_Z$ and variance $\sigma_Z^2$, we approximate $Z$ by a log-normal distribution with matching first two moments, i.e., $Z \sim
\operatorname{LogNormal}
\!\left(
    Z;
    \mu_{\log Z},
    \sigma_{\log Z}^2
\right)$. By definition of the log-normal distribution, $\mu_Z=\exp\left(\mu_{\log Z}+\frac{1}{2}\sigma_{\log Z}^2\right)$ and $\sigma_Z^2=\left(\exp(\sigma_{\log Z}^2)-1\right)\exp\left(2\mu_{\log Z}+\sigma_{\log Z}^2\right)$. Solving for $\mu_{\log Z}$ and $\sigma_{\log Z}^2$, we get $\sigma_{\log Z}^2=\log\left(1+\frac{\sigma_Z^2}{\mu_Z^2}\right)$ and $\mu_{\log Z}=\log(\mu_Z)-\frac{1}{2}\sigma_{\log Z}^2$. Therefore, we have $\log(Z) \sim
\mathcal N\!\left(
    \cdot;
    \mu_{\log Z},
    \sigma_{\log Z}^2
\right)$ under the log-normal approximation. We apply this property to get the moments of $\log(1 + Y)$ and $\log\left(1 - \frac{Y}{\alpha}\right)$ and compute final moments of the function from Equation \ref{eq:activation-inverse}.

\subsubsection{Activation derivative messages}
\label{appendix:activation-derivative-messages}

The derivative of $g(z)$ is defined as
\begin{equation*}
g'(z)
=
\begin{cases}
\beta e^{\beta z}, & z>0,\\
\alpha\beta e^{-\beta z}, & z\le0.
\end{cases}
\end{equation*}

\paragraph{Forward message}
Using again Lemma~\ref{lemma:conditional_lognormal}, the forward moments are computed directly as
\begin{equation*}
\begin{aligned}
\mathbb{E}[g'(Z)]
&=
\beta
\mathbb{E}\!\left[
e^{\beta Z}\mathbf{1}_{\{Z>0\}}
\right]
+
\alpha\beta
\mathbb{E}\!\left[
e^{-\beta Z}\mathbf{1}_{\{Z\le0\}}
\right],
\\
\mathbb{E}[g'(Z)^2]
&=
\beta^2
\mathbb{E}\!\left[
e^{2\beta Z}\mathbf{1}_{\{Z>0\}}
\right]
+
\alpha^2\beta^2
\mathbb{E}\!\left[
e^{-2\beta Z}\mathbf{1}_{\{Z\le0\}}
\right].
\end{aligned}
\end{equation*}

\paragraph{Backward message}
Given $g'(Z)=D$, for the positive branch, $Z>0$ implies $D>\beta$, while for the negative branch, $Z\leq0$ implies $D\geq\alpha\beta$. Solving each branch for $Z$ gives
\begin{equation*}
\begin{aligned}
(g'_+)^{-1}(D) &= \dfrac{1}{\beta}\log\!\left(\dfrac{D}{\beta}\right), && D>\beta,\\[2mm]
(g'_-)^{-1}(D) &= -\dfrac{1}{\beta}\log\!\left(\dfrac{D}{\alpha\beta}\right), && D\geq\alpha\beta.
\end{aligned}
\end{equation*}
Unlike the forward transformation, $D$ does not uniquely identify the branch, since its ranges may overlap. For this reason, the residual-factor updates are derived without requiring an inverse of $g'$; see Appendices~\ref{appendix:advection} and \ref{appendix:kpp}.

\subsection{Derivations for the residual factor}
\label{appendix:residual-derivations}
In this section, we show the derivations of the residual factor $f_R$ for the advection equation (Appendix \ref{appendix:advection}) and the Fisher-KPP equation (Appendix \ref{appendix:kpp}). All derivations are performed for the shallow neural architecture used throughout this work, defined in Section \ref{section:approach}.

\subsubsection{Advection equation}
\label{appendix:advection}

The advection equation has a generic form given by Equation~\ref{eq:first_order_pde}. The physical meaning of the advection equation is the transport of a quantity by a given velocity field. For the shallow architecture defined in Section~\ref{section:approach}, the chain rule gives
\begin{equation*}
f_R^\prime(\mW,\vx) =  \vb^\top\nabla_{\vx}f_U(\vx;\vtheta)
=
g'(z)\mW^\top
\sum_{i=1}^{d} b_i\frac{\partial\phi(\vx)}{\partial x_i},
\qquad
z=\mW^\top\phi(\vx).
\end{equation*}

\paragraph{Forward message}

Starting from Equation~\ref{eq:advection_residual}, we compute the residual expectation as

\begin{equation*}
\mathbb{E}[c(\mW,\vx)]
=
\mathbb{E}[g'(z)\mW^\top]
\sum_{i=1}^{d} b_i\frac{\partial\phi(\vx)}{\partial x_i}
=
\sum_{i=1}^{d} b_i\frac{\partial\phi(\vx)}{\partial x_i} \cdot
\left(
\mathbb{E}[g'(z)]\mathbb{E}[\mW^\top]
+
\operatorname{Cov}(g'(z),\mW^\top)
\right).
\end{equation*}

The covariance term $\operatorname{Cov}(g'(z),\mW^\top)$ is computed as follows. Stein's lemma states that for two jointly Gaussian variables $\eta_1$ and $\eta_2$ and a function $\chi(.)$, the following equality holds:

\begin{equation}
\label{eq:stein-lemma}
\operatorname{Cov}(\eta_1,\chi(\eta_2))
=
\operatorname{Cov}(\eta_1,\eta_2)
\mathbb{E}[\chi'(\eta_2)].
\end{equation}

Since $W_k$ and $z$ are jointly Gaussian, we set $\eta_1=W_k$, $\eta_2=z$, and replace $\chi$ by  our $g'$. Hence,
\begin{equation*}
\begin{aligned}
\operatorname{Cov}(W_k,g'(z))
&=
\operatorname{Cov}(W_k,z)
\mathbb{E}[g''(z)]
=
\operatorname{Cov}
\left(
W_k,
\sum_j W_j\phi_j(\vx)
\right)
\mathbb{E}[g''(z)]
\\
&=
\sum_{j=1}^{P}
\phi_j(\vx)\operatorname{Cov}(W_k,W_j)
\mathbb{E}[g''(z)]
=
\phi_k(\vx)\sigma_k^2
\mathbb{E}[g''(z)].
\end{aligned}
\end{equation*}

As $c(\mW,\vx)$ is a non-linear function of correlated Gaussians, it does not have a closed-form variance expression, and we use the first-order Taylor expansion for $c(\mW,\vx)$ to get the variance approximation. The generic formula for this expansion is

\begin{equation}
\label{eq:first-order-approx}
F(X,Y)
\approx
F\bigl(\mathbb{E}[X],\mathbb{E}[Y]\bigr)
+
\left.
\frac{\partial F}{\partial X}
\right|_{(\mathbb{E}[X],\mathbb{E}[Y])}
\left(X-\mathbb{E}[X]\right)
+
\left.
\frac{\partial F}{\partial Y}
\right|_{(\mathbb{E}[X],\mathbb{E}[Y])}
\left(Y-\mathbb{E}[Y]\right).
\end{equation}

For simplicity, let $r(\vx) := \mW^\top \sum_{i=1}^{d} b_i\frac{\partial\phi(\vx)}{\partial x_i}$.
Then, $c(\mW,\vx)=g'(z)r(\vx)$ and we can approximate it as

\begin{equation}
\begin{aligned}
\label{eq:residual-approximation}
c(\mW,\vx)
&=
g'(z)r(\vx)
\\
&\approx
g'(\mathbb{E}[z])\mathbb{E}[r(\vx)]
+
g''(\mathbb{E}[z])\mathbb{E}[r(\vx)]
\left(z-\mathbb{E}[z]\right)
\\
&\quad+
g'(\mathbb{E}[z])
\left(r(\vx)-\mathbb{E}[r(\vx)]\right)
\\
&=
g'(\mathbb{E}[z])r(\vx)
+
g''(\mathbb{E}[z])\mathbb{E}[r(\vx)]
\left(z-\mathbb{E}[z]\right).
\end{aligned}
\end{equation}

Thus, the second moment approximation is

\begin{equation*}
\begin{aligned}
\mathbb{E}[(c(\mW,\vx)))^2]
&\approx
g'(\mathbb{E}[z])^2
\mathbb{E}[r(\vx)^2]
\\
&\quad+
\mathbb{E}[r(\vx)]^2
g''(\mathbb{E}[z])^2
\operatorname{Var}(z)
\\
&\quad+
2g'(\mathbb{E}[z])
\mathbb{E}[r(\vx)]
g''(\mathbb{E}[z])
\operatorname{Cov}(r(\vx),z),
\end{aligned}
\end{equation*}

where the covariance term can be computed straightforwardly.

\paragraph{Backward message}
In the backward update, we send a message to each $W_k$. Since
$z = \mW^\top\phi(\vx) = \sum_j W_j[\phi(\vx)]_j$
depends on $W_k$, the residual is nonlinear in $W_k$ and cannot be solved for $W_k$ directly; we therefore use the same first-order approximation as in the forward direction, defined in Equation~\ref{eq:residual-approximation}. We solve Equation~\ref{eq:first-order-approx} for $W_k$, assuming all moments and derivatives are already evaluated and therefore constitute fixed coefficients. Then, the coefficient before $W_k$ is

\begin{equation*}
A_k
=
\left(
\sum_{i=1}^{d} b_i
\left[
\frac{\partial\phi(\vx)}{\partial x_i}
\right]_k
\right)
g'(\mathbb{E}[z])
+
[\phi(\vx)]_k
\mathbb{E}[r(\vx)]
g''(\mathbb{E}[z]).
\end{equation*}

Then, collecting all other terms, we can express $W_k$ as

\begin{equation*}
\begin{aligned}
W_k
&\approx
\frac{1}{A_k}
\Bigg[
R
-
g'(\mathbb{E}[z])
\sum_{j\neq k}
\left(
\sum_{i=1}^{d}b_i
\left[
\frac{\partial\phi(\vx)}{\partial x_i}
\right]_j
\right)W_j
\\
&\quad
-
\mathbb{E}\left[
\mW^\top
\sum_{i=1}^{d}b_i
\frac{\partial\phi(\vx)}{\partial x_i}
\right]
g''(\mathbb{E}[z])
\left(
\sum_{j\neq k}
[\phi(\vx)]_jW_j
-
\mathbb{E}[z]
\right)
\Bigg].
\end{aligned}
\end{equation*}

Given that the incoming residual message is independent of the incoming weights and $\mathbb{E}[R]=0$, moment matching gives the following results:
\begin{equation*}
\begin{aligned}
\mathbb{E}[W_k]
&\approx
\frac{1}{A_k}
\Bigg[
-
g'(\mathbb{E}[z])
\sum_{j\neq k}
\left(
\sum_{i=1}^{d}b_i
\left[
\frac{\partial\phi(\vx)}{\partial x_i}
\right]_j
\right)\mu_j
\\
&\quad
-
\mathbb{E}\left[
\mW^\top
\sum_{i=1}^{d}b_i
\frac{\partial\phi(\vx)}{\partial x_i}
\right]
g''(\mathbb{E}[z])
\left(
\sum_{j\neq k}
[\phi(\vx)]_j\mu_j
-
\mathbb{E}[z]
\right)
\Bigg], \\
\operatorname{Var}(W_k)
&\approx
\frac{1}{A_k^2}
\Bigg[
\epsilon_C^2
+
g'(\mathbb{E}[z])^2
\sum_{j\neq k}
\left(
\sum_{i=1}^{d}b_i
\left[
\frac{\partial\phi(\vx)}{\partial x_i}
\right]_j
\right)^2
\sigma_j^2
\\
&\quad
+
\mathbb{E}\left[
\mW^\top
\sum_{i=1}^{d}b_i
\frac{\partial\phi(\vx)}{\partial x_i}
\right]^2
g''(\mathbb{E}[z])^2
\sum_{j\neq k}
[\phi(\vx)]_j^2\sigma_j^2
\\
&\quad
+
2
g'(\mathbb{E}[z])
\mathbb{E}\left[
\mW^\top
\sum_{i=1}^{d}b_i
\frac{\partial\phi(\vx)}{\partial x_i}
\right]
g''(\mathbb{E}[z])
\sum_{j\neq k}
\left(
\sum_{i=1}^{d}b_i
\left[
\frac{\partial\phi(\vx)}{\partial x_i}
\right]_j
\right)
[\phi(\vx)]_j
\sigma_j^2
\Bigg].
\end{aligned}
\end{equation*}

\subsubsection{Fisher-KPP equation}
\label{appendix:kpp}

The Fisher-KPP equation is a reaction--diffusion PDE that extends the idea of transport and describes how a quantity spreads through diffusion while also growing locally. It is expressed through the following PDE:
\begin{equation}
\label{eq:fisher_residual}
f_R^\prime(t,s)
:=
\frac{\partial f}{\partial t}
-
\nu\frac{\partial^2 f}{\partial s^2}
-
\rho f(1-f),
\end{equation}
where $\nu>0$ is a given diffusion coefficient, $\rho>0$ is the logistic growth rate, and $f=f_U(\vx;\vtheta)$ denotes the inferred population density.

\paragraph{Forward message}
The expectation follows by taking the expectation of each term in
Equation~\ref{eq:fisher_residual}:
\begin{equation*}
\begin{aligned}
\mathbb E[c(t,s)]
={}&
\mathbb E\left[
g'(z)\mW^\top\frac{\partial\phi}{\partial t}
\right]
-
\nu\mathbb E\left[
g'(z)\mW^\top\frac{\partial^2\phi}{\partial s^2}
\right]
\\
&-
\nu\mathbb E\left[
g''(z)
\left(
\mW^\top\frac{\partial\phi}{\partial s}
\right)^2
\right]
-\rho\mathbb E[g(z)]
+\rho\mathbb E[g(z)^2].
\end{aligned}
\end{equation*}

Using Stein's lemma from Equation~\ref{eq:stein-lemma}, for any deterministic vector $\va$,
\begin{equation*}
\begin{aligned}
\mathbb E[g'(z)\mW^\top\va]
&=
\sum_k a_k\,\mathbb E[W_k g'(z)]
\\
&=
\sum_k a_k
\left(
\mathbb E[W_k]\mathbb E[g'(z)]
+
\operatorname{Cov}(W_k,g'(z))
\right)
\\
&=
\sum_k a_k
\left(
\mathbb E[W_k]\mathbb E[g'(z)]
+
\phi_k(\vx)\operatorname{Var}(W_k)\mathbb E[g''(z)]
\right).
\end{aligned}
\end{equation*}

Substituting $\frac{\partial\phi}{\partial t}$ and
$\frac{\partial^2\phi}{\partial s^2}$ for $\va$, we can compute the first two terms.
For the next quadratic derivative term, let
\[
\zeta_s
=
\mW^\top\frac{\partial\phi}{\partial s}.
\]
Its first two moments are
\begin{equation*}
\mathbb E[\zeta_s]
=
\sum_k
\mathbb E[W_k]\frac{\partial\phi_k}{\partial s},
\qquad
\mathbb E[\zeta_s^2]
=
\mathbb E[\zeta_s]^2
+
\sum_k
\operatorname{Var}(W_k)
\left(
\frac{\partial\phi_k}{\partial s}
\right)^2.
\end{equation*}
Since $z$ and $\zeta_s$ depend on the same weights, they are generally correlated.
For simplicity, we introduce an additional moment-factorization approximation by neglecting the dependence between $g''(z)$ and $\zeta_s^2$:
\begin{equation*}
\mathbb E[g''(z)\zeta_s^2]
\approx
\mathbb E[g''(z)]\,\mathbb E[\zeta_s^2].
\end{equation*}

Thus, we approximate the residual mean as
\begin{equation*}
\begin{aligned}
\mathbb E[c(t,s)]
\approx{}&
\sum_k
\frac{\partial\phi_k}{\partial t}
\left(
\mathbb E[W_k]\mathbb E[g'(z)]
+
\phi_k\operatorname{Var}(W_k)\mathbb E[g''(z)]
\right)
\\
&-
\nu\sum_k
\frac{\partial^2\phi_k}{\partial s^2}
\left(
\mathbb E[W_k]\mathbb E[g'(z)]
+
\phi_k\operatorname{Var}(W_k)\mathbb E[g''(z)]
\right)
\\
&-
\nu\mathbb E[g''(z)]
\left[
\left(
\sum_k
\mathbb E[W_k]\frac{\partial\phi_k}{\partial s}
\right)^2
+
\sum_k
\operatorname{Var}(W_k)
\left(
\frac{\partial\phi_k}{\partial s}
\right)^2
\right]
\\
&-
\rho\mathbb E[g(z)]
+
\rho\mathbb E[g(z)^2].
\end{aligned}
\end{equation*}

To approximate the variance, as in the derivations for the advection equation in
Section~\ref{appendix:advection}, we use a first-order approximation of the
complete residual around the expected weights, as in
Equation~\ref{eq:first-order-approx}. Writing $R=f_R(\vx,\mW)$,
\begin{equation*}
R
\approx
R(\mathbb E[\mW])
+
\sum_k
\left.
\frac{\partial R}{\partial W_k}
\right|_{\mW=\mathbb E[\mW]}
\left(
W_k-\mathbb E[W_k]
\right).
\end{equation*}
Since the weights are represented by independent Gaussian messages,
\begin{equation*}
\operatorname{Var}(R)
\approx
\sum_k
\left(
\left.
\frac{\partial R}{\partial W_k}
\right|_{\mW=\mathbb E[\mW]}
\right)^2
\operatorname{Var}(W_k).
\end{equation*}

Differentiating the residual gives
\begin{equation*}
\begin{aligned}
\frac{\partial R}{\partial W_k}
={}&
g'(z)\frac{\partial\phi_k}{\partial t}
+
g''(z)\phi_k
\mW^\top\frac{\partial\phi}{\partial t}
\\
&-
\nu g'(z)\frac{\partial^2\phi_k}{\partial s^2}
-
\nu g''(z)\phi_k
\mW^\top\frac{\partial^2\phi}{\partial s^2}
\\
&-
\nu g'''(z)\phi_k
\left(
\mW^\top\frac{\partial\phi}{\partial s}
\right)^2
\\
&-
2\nu g''(z)
\left(
\mW^\top\frac{\partial\phi}{\partial s}
\right)
\frac{\partial\phi_k}{\partial s}
\\
&+
\rho\phi_k g'(z)
\bigl(2g(z)-1\bigr).
\end{aligned}
\end{equation*}

The moments of $g(z)$ and $g'(z)$ are given in Appendices~\ref{appendix:activation-messages} and \ref{appendix:activation-derivative-messages}, respectively; the computation of
$g''(z)$ and $g'''(z)$ changes only by the corresponding constant factors.

\paragraph{Backward message}
In the backward update, we apply the same local linearization used for the advection equation in Appendix~\ref{appendix:advection}. The coefficient $A_k$ is defined by evaluating the residual derivative at the expected weights:
\begin{equation*}
A_k
=
\left.
\frac{\partial R}{\partial W_k}
\right|_{\mW=\mathbb E[\mW]}.
\end{equation*}

For notation simplicity, define
\begin{equation*}
\begin{aligned}
\mathbb E[z]
=
\mathbb E[\mW]^\top\phi,
&\quad
\mathbb E[\zeta_t]
=
\mathbb E[\mW]^\top
\frac{\partial\phi}{\partial t},
\\
\mathbb E[\zeta_{ss}]
=
\mathbb E[\mW]^\top
\frac{\partial^2\phi}{\partial s^2},
&\quad
\mathbb E[\zeta_s]
=
\mathbb E[\mW]^\top
\frac{\partial\phi}{\partial s}.
\end{aligned}
\end{equation*}

Then, we can expand $A_k$ as
\begin{equation*}
\begin{aligned}
A_k
={}&
g'(\mathbb E[z])
\frac{\partial\phi_k}{\partial t}
+
g''(\mathbb E[z])
\phi_k
\mathbb E[\zeta_t]
\\
&-
\nu g'(\mathbb E[z])
\frac{\partial^2\phi_k}{\partial s^2}
-
\nu g''(\mathbb E[z])
\phi_k
\mathbb E[\zeta_{ss}]
\\
&-
\nu g'''(\mathbb E[z])
\phi_k
\mathbb E[\zeta_s]^2
-
2\nu g''(\mathbb E[z])
\mathbb E[\zeta_s]
\frac{\partial\phi_k}{\partial s}
\\
&+
\rho\phi_k
g'(\mathbb E[z])
\left(
2g(\mathbb E[z])-1
\right).
\end{aligned}
\end{equation*}

Then, $W_k$ can be approximated as
\begin{equation*}
W_k
\approx
\mathbb E[W_k]
+
\frac{1}{A_k}
\left[
R
-
\mathbb E[R]
-
\sum_{j\neq k}
A_j
\left(
W_j-\mathbb E[W_j]
\right)
\right].
\end{equation*}

Given target $R\sim\mathcal N(\cdot;0,\epsilon_C^2)$, we substitute $\mathbb E[R]$ and
$\operatorname{Var}(R)$ accordingly and moment match the $W_k$ approximation in the same way as for the backward message for the advection equation in Appendix~\ref{appendix:advection}.

\subsection{Experimental details}
\label{appendix:experimental_details}

\paragraph{Data} Trajectories are drawn once from a fixed-seed permutation of each PDEBench~\citep{takamoto2022pdebench} setting; the trajectory used during development is excluded. Advection uses ten trajectories across five speeds each; the time cost of these experiments is mostly dominated by the HMC and VI references. Fisher-KPP uses five trajectories in each of its 16 settings. Its stored trajectories freeze at non-physical values once the field homogenizes (a float32 artifact of the dataset generator), so Fisher-KPP errors are measured against a float64 re-solve of the PDE from the same initial condition.

\paragraph{Shared setup} We use $1024$ labeled points at $t=0$ and $48K_t$ stratified collocation points, where $K_t$ is the number of temporal Fourier frequencies in the basis. For each trajectory, all methods share the basis, rotation, factor locations, and seeds, thereby targeting the same posterior. Message passing visits the factors in a shuffled order each epoch and stops when the largest per-factor Kullback--Leibler (KL) divergence falls below $10^{-5}$, or after 300 epochs.

\paragraph{Uncertainty metrics} 
\label{appendix:uncertainty-metrics}
Each method yields a Gaussian predictive $\mathcal N(\cdot;\mu_j,\sigma_j^2)$ of the network output at every point $(x_j,t_j)$ of the evaluation grid: our marginal moments, or the sample mean and variance over posterior draws for HMC and VI. With the exact solution $u_j$ and $z_j=(u_j-\mu_j)/\sigma_j$, the coverage of the nominal 90\,\% interval and the closed-form Gaussian CRPS~\citep{gneiting2007strictly} are
\begin{equation*}
\mathrm{Cov}_{90} = \frac{1}{G}\sum_{j=1}^{G}\mathds{1}\left[\,|z_j|\le\Phi^{-1}(0.95)\right],
\quad
\mathrm{CRPS} = \frac{1}{G}\sum_{j=1}^{G}\sigma_j\!\left[z_j\bigl(2\Phi(z_j)-1\bigr)+2\varphi(z_j)-\tfrac{1}{\sqrt\pi}\right],
\end{equation*}
where $\Phi$ and $\varphi$ are the standard normal cumulative density function and density, $\Phi^{-1}(0.95)\approx1.645$, and the averages run over the $G$ grid points of all runs. For a point prediction ($\sigma_j\to0$) CRPS reduces to the absolute error $|u_j-\mu_j|$.

\paragraph{Reaction continuation (Fisher-KPP)} The logistic reaction adds a spurious posterior mode at $f\equiv0$, which message passing and gradient-based training reach from their default initializations at a strong reaction. This is a common failure mode for learning PDEs~\cite{krishnapriyan2021characterizing}. All methods therefore train in four passes at reaction coefficient $\lambda\rho$, $\lambda\in\{0.25,0.5,0.75,1\}$, each warm-started from the previous one, within unchanged iteration budgets. This roughly doubles message passing training time; reported times include it.

\paragraph{Baselines} The SGD baseline, VI and the HMC warm start use Adam (learning rate $3\times10^{-4}$, at most $30{,}000$ iterations, split evenly over the passes on Fisher-KPP), halving the rate on loss plateaus (less than $0.1\,\%$ improvement over 500 iterations) up to four times and stopping at the next plateau. VI is mean-field Gaussian with eight reparameterized samples per iteration.

\paragraph{HMC reference} Four chains of $1{,}000$ iterations use a dense Gauss--Newton mass matrix at the maximum a posteriori (MAP) estimate (the polished Adam warm start; on Fisher-KPP, the better of the warm starts from zero and from the message passing mean). The step size is adapted by dual averaging to acceptance $0.8$ over 250 discarded warm-up iterations, and the number of leapfrog steps per proposal is drawn uniformly from $\{1,\dots,20\}$. Every coordinate of all 50 advection and 80 Fisher-KPP reference runs reaches split-$\hat R<1.05$ and bulk effective sample size $>100$.

\paragraph{Width ablation} For this ablation, the wider model is $f_H(\vx)=\sum_{h=1}^{H} a_h g(z_h)$ with $z_h=\mW_h^\top\phi(\vx)$ and fixed, distinct readout weights $a_h$. Models with $H\in\{2,3,4\}$ hidden units keep everything else of the single-unit model. They run on all 50 advection runs and, on Fisher-KPP, on all 80 runs at $H=2$ and on 33 ($H=3$) and 19 ($H=4$) runs covering all 16 settings.

\paragraph{Hardware and timing} All experiments ran on one 8-core AMD EPYC (Genoa) central processing unit (CPU) with 16\,GB random-access memory (RAM) and no graphics processing unit (GPU) (Julia 1.12.6). Reported training times are cold-start and include just-in-time compilation; steady-state inference times are the minimum over seven repetitions (three for HMC) after a warm-up call on the otherwise idle machine.
\subsection{Additional Evaluations}
\label{appendix:additional-evaluations}
We provide additional experimental evaluations for the advection and Fisher-KPP equations and the PDE residual prior.
\subsubsection{Advection equation}
\label{appendix:advection-results}

This appendix includes additional results for the advection equation that are mentioned in Section \ref{section:evaluation} before. Figure \ref{fig:mp_vs_pdebench_scatter_grid} compares inferred posteriors between our approach and the non-Bayesian PINN baseline. Figure \ref{fig:uq}(a) demonstrates a monotonic pattern in the uncertainty increase over time; it also shows that the rate at which uncertainty increases depends on the wave speed \(b\), indicating that the inferred uncertainty responds to changes in the PDE dynamics as well as to the distance from the observed initial condition.
\begin{figure}
    \centering
    \includegraphics[width=.8\linewidth]{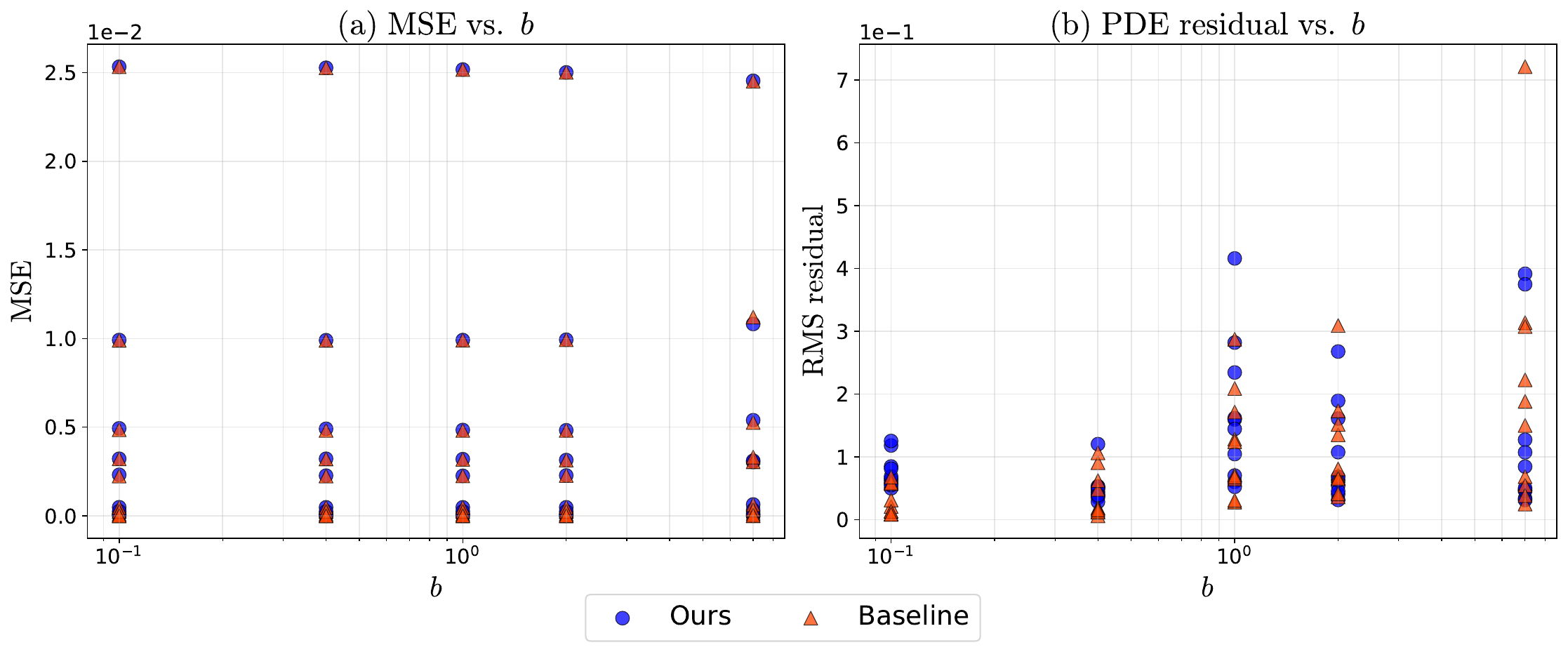}
    \caption{Comparison of our message passing approach (ours) and the SGD baseline for the advection equation; each point refers to one trajectory. \textbf{(a)} mean squared error (MSE) of the predictive mean against the exact solution; \textbf{(b)} root mean square (RMS) of the PDE residual $R$. Both methods reach errors of the same order at every wave speed.}
    \label{fig:mp_vs_pdebench_scatter_grid}
\end{figure}
\begin{figure}
    \centering
    \includegraphics[width=1.\linewidth]{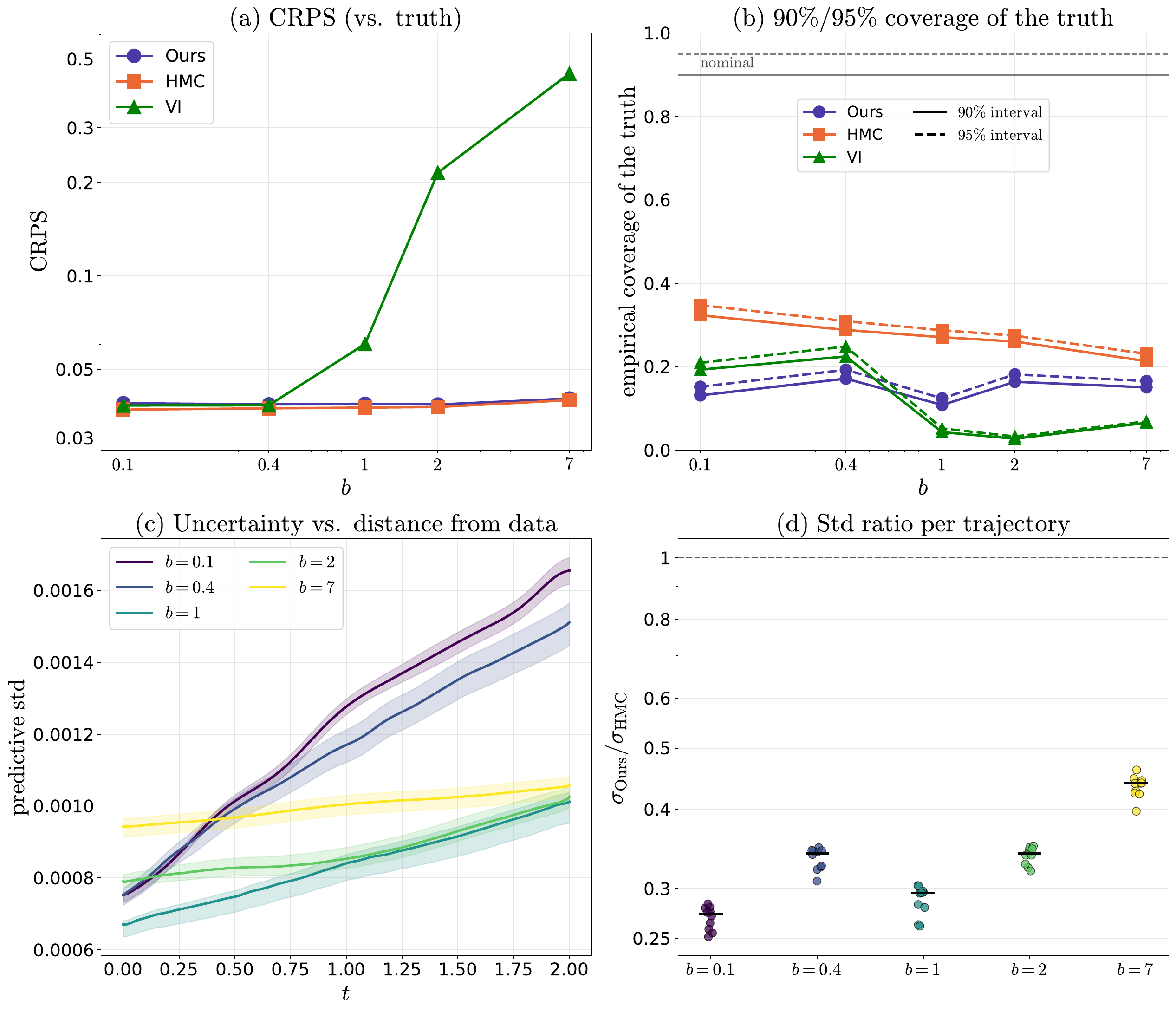}
    \caption{Calibration and uncertainty structure for the advection equation, ten trajectories per wave speed $b$. \textbf{(a)} CRPS against the exact solution; \textbf{(b)} empirical coverage of the exact solution by the nominal 90\,\% (solid) and 95\,\% (dashed) predictive intervals; \textbf{(c)} our mean predictive standard deviation over time $t$ (band: minimum to maximum over trajectories), growing with the distance from the observed initial condition at a rate that depends on $b$; \textbf{(d)} per-trajectory median ratio of our predictive standard deviation to that of HMC (bar: median over trajectories), roughly constant within each $b$. Ours and HMC share accuracy and the uncertainty pattern, whereas VI degrades for $b\geq1$.}
\label{fig:uq}
\end{figure}

Figures \ref{fig:bayesian_comparison_ours_vi} and \ref{fig:bayesian_comparison_vi_hmc} support the claim from Section \ref{fig:bayesian_comparison} that our inferred means are close to those from HMC and VI, but our inferred variance has a more similar structure to HMC than VI does.

\begin{figure}
    \centering
    \includegraphics[width=.95\linewidth]{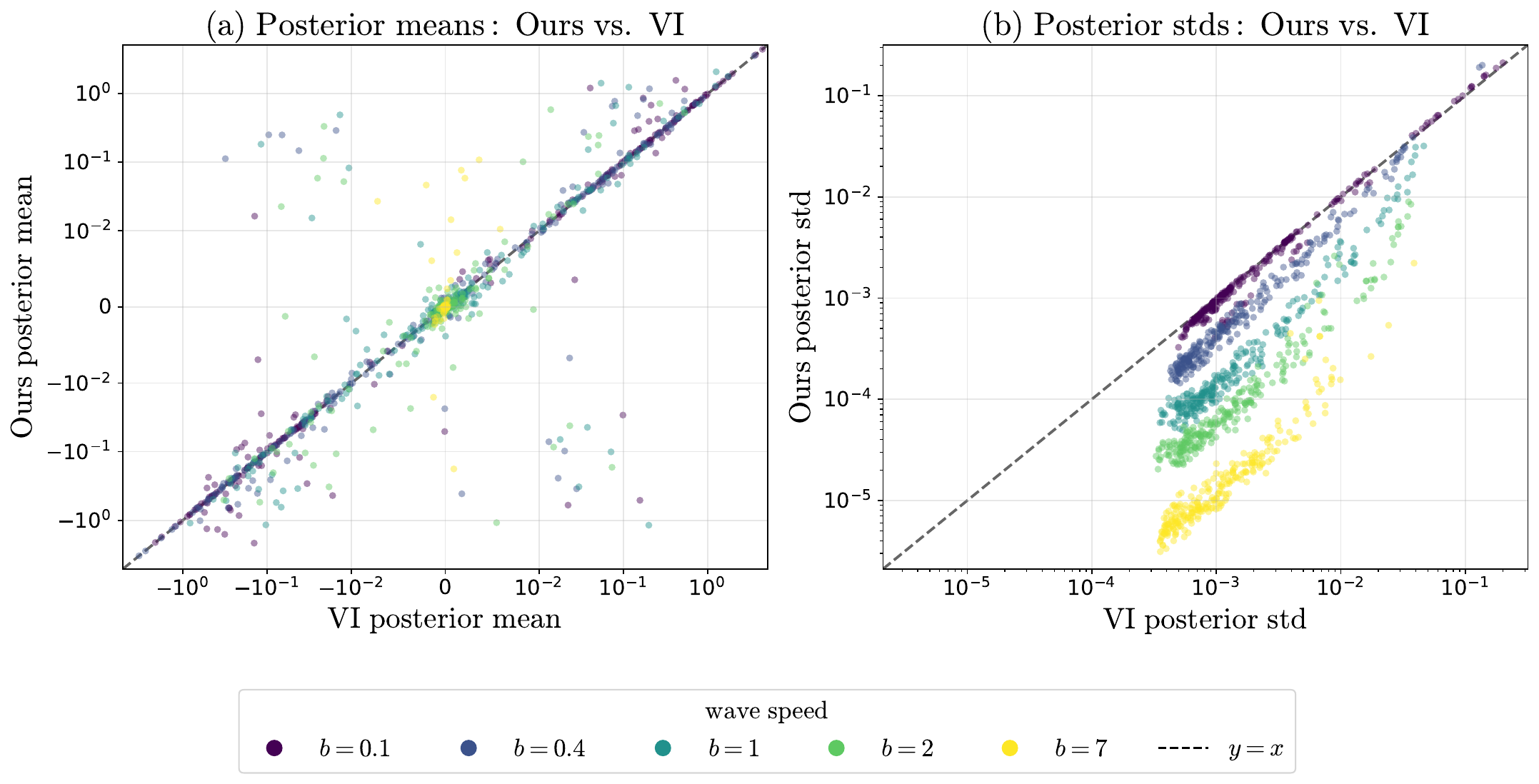}
    \caption{Comparison of our approach (ours) to VI, implemented based on \cite{yang2021b}, for the advection equation. We map posterior weights of VI (x-axis) to our approach (y-axis) - points in \textbf{(a)} refer to the mean values, points in \textbf{(b)} refer to the standard deviations; the means largely agree, while the VI standard deviations deviate from ours increasingly with the wave speed $b$. For readability, 25 random weights per trajectory are shown.}
    \label{fig:bayesian_comparison_ours_vi}
\end{figure}

\begin{figure}
    \centering
    \includegraphics[width=.95\linewidth]{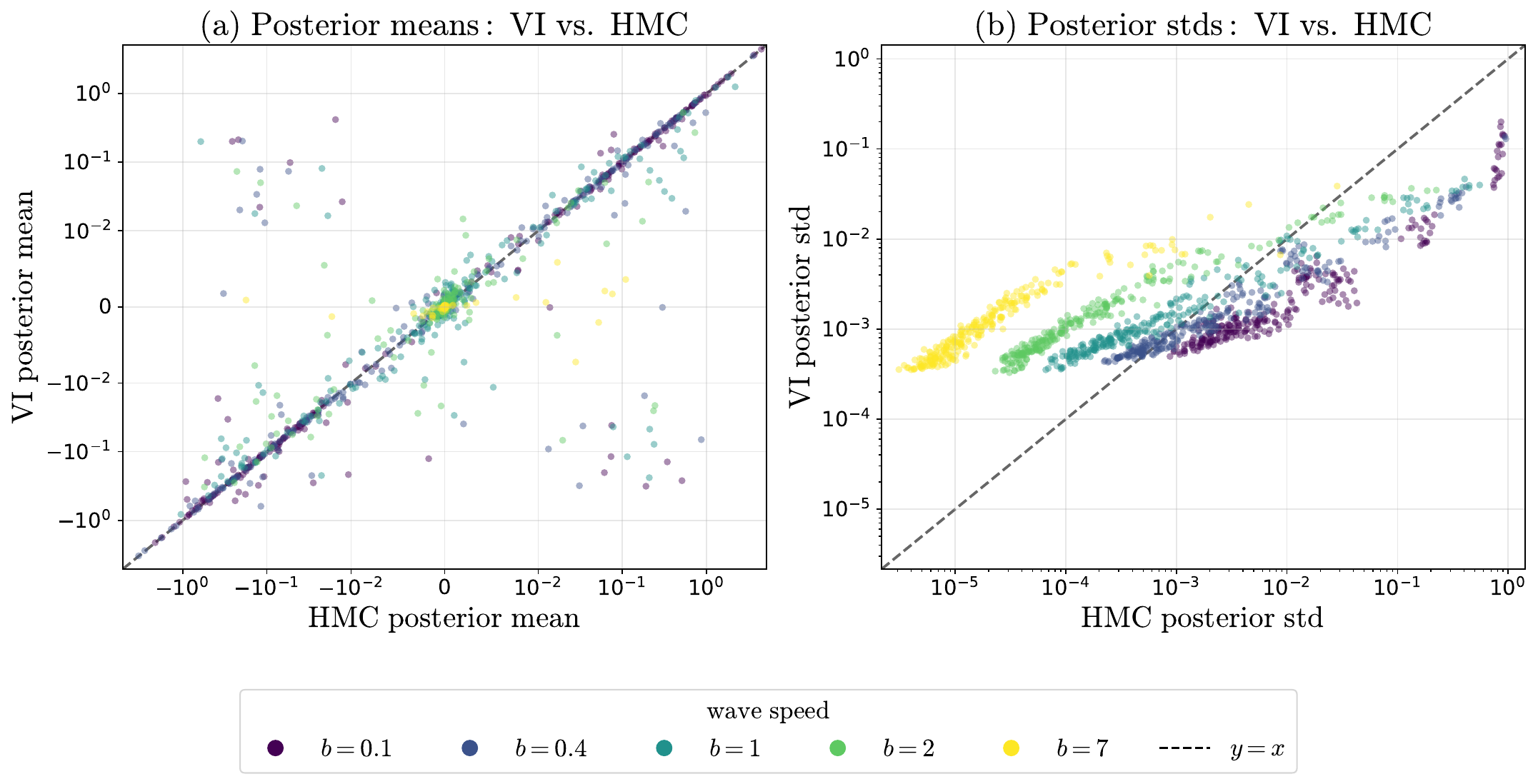}
    \caption{Comparison of VI to HMC, both implemented based on \cite{yang2021b}, for the advection equation. We map posterior weights of HMC (x-axis) to VI (y-axis) - points in \textbf{(a)} refer to the mean values, points in \textbf{(b)} refer to the standard deviations; VI underestimates the standard deviations at $b\leq0.4$ and overestimates them at $b\geq2$, so it is further from HMC than our approach (Figure~\ref{fig:bayesian_comparison}). For readability, 25 random weights per trajectory are shown.}
    \label{fig:bayesian_comparison_vi_hmc}
\end{figure}

Figure~\ref{fig:wider_model} provides an ablation study on wider architectures with the increasing number of neurons. Among other results, it demonstrates that increasing dimensionality does not substantially increase training time.

\begin{figure}
    \centering
    \includegraphics[width=1.\linewidth]{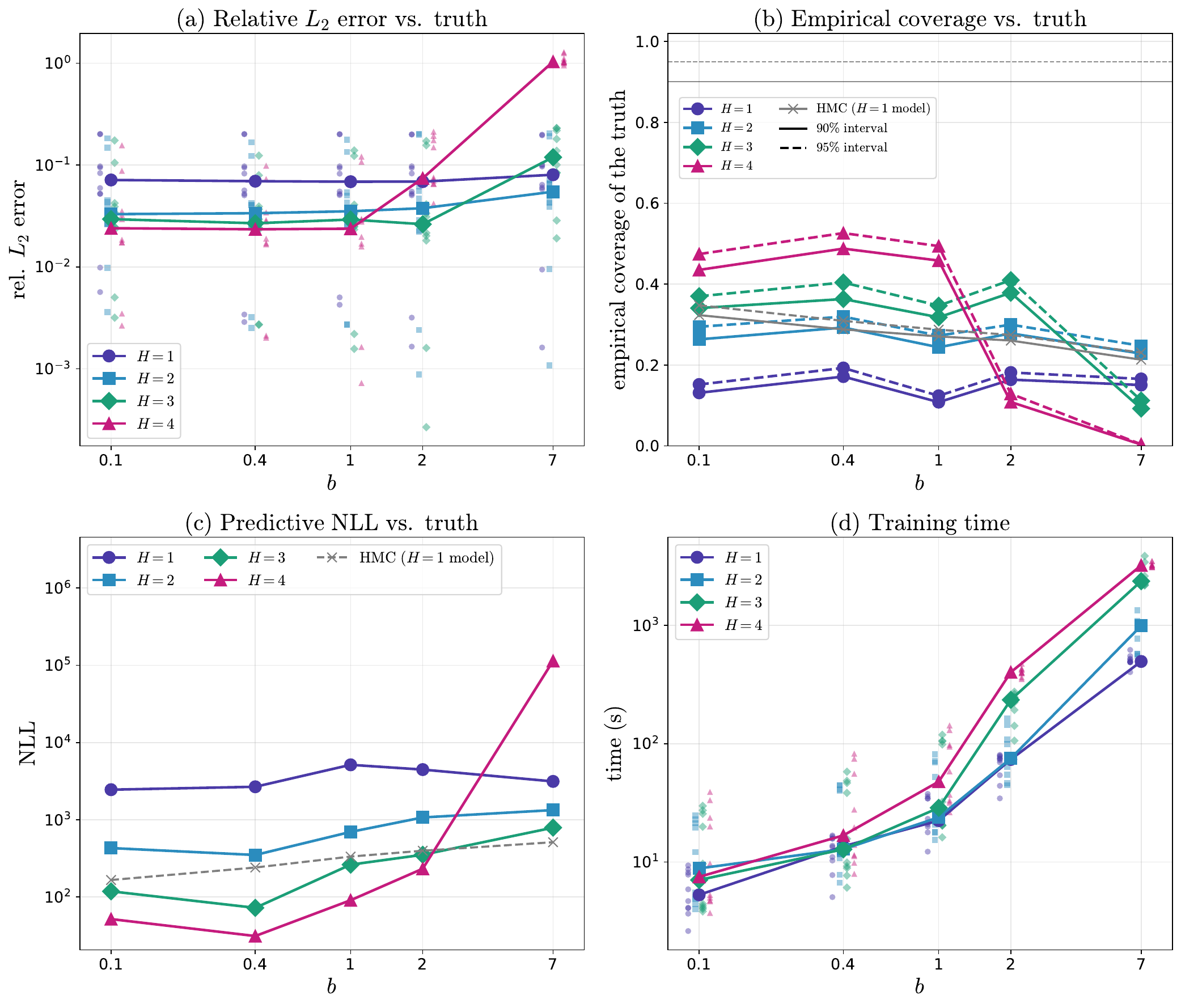}
    \caption{Neural layer width ablation for the advection equation: $H\in\{1,2,3,4\}$ neurons in the hidden layer. \textbf{(a)} Relative $L_2$ error of the predictive mean (points: trajectories; lines: medians); \textbf{(b)} empirical coverage of the exact solution by the nominal 90\,\% (solid) and 95\,\% (dashed) predictive intervals; \textbf{(c)} predictive negative log-likelihood (NLL) of the exact solution, mean over trajectories; \textbf{(d)} training time (points: trajectories; lines: medians). HMC (grey) refers to the $H=1$ model. Wider layers reduce the error and improve calibration for $b\leq1$ at a higher training cost; at larger $b$, $H\geq3$ degrades, and $H=4$ fails at $b=7$.}
\label{fig:wider_model}
\end{figure}
\subsubsection{Fisher-KPP equation}
\label{appendix:fisher-results}
This appendix presents additional results for the Fisher-KPP equation that are discussed in the main text but included here due to space limitations. Figures \ref{fig:fisher-pdebench} and \ref{fig:fisher-simple} show predicted means and variances for two different random trajectories; as in the advection results, the means reflect dynamics close to the exact solutions. For one of the initial conditions, the prediction error increases with the growth rate $\rho$, but configurations with high growth rates are already known to be problematic and are described in \cite{krishnapriyan2021characterizing} as a source of failure modes. The predicted variance shows the following pattern: it is especially large when a physical process is active, but otherwise it grows over time as we move further from the observed initial condition.

\begin{figure}
    \centering
    \includegraphics[width=1.\linewidth]{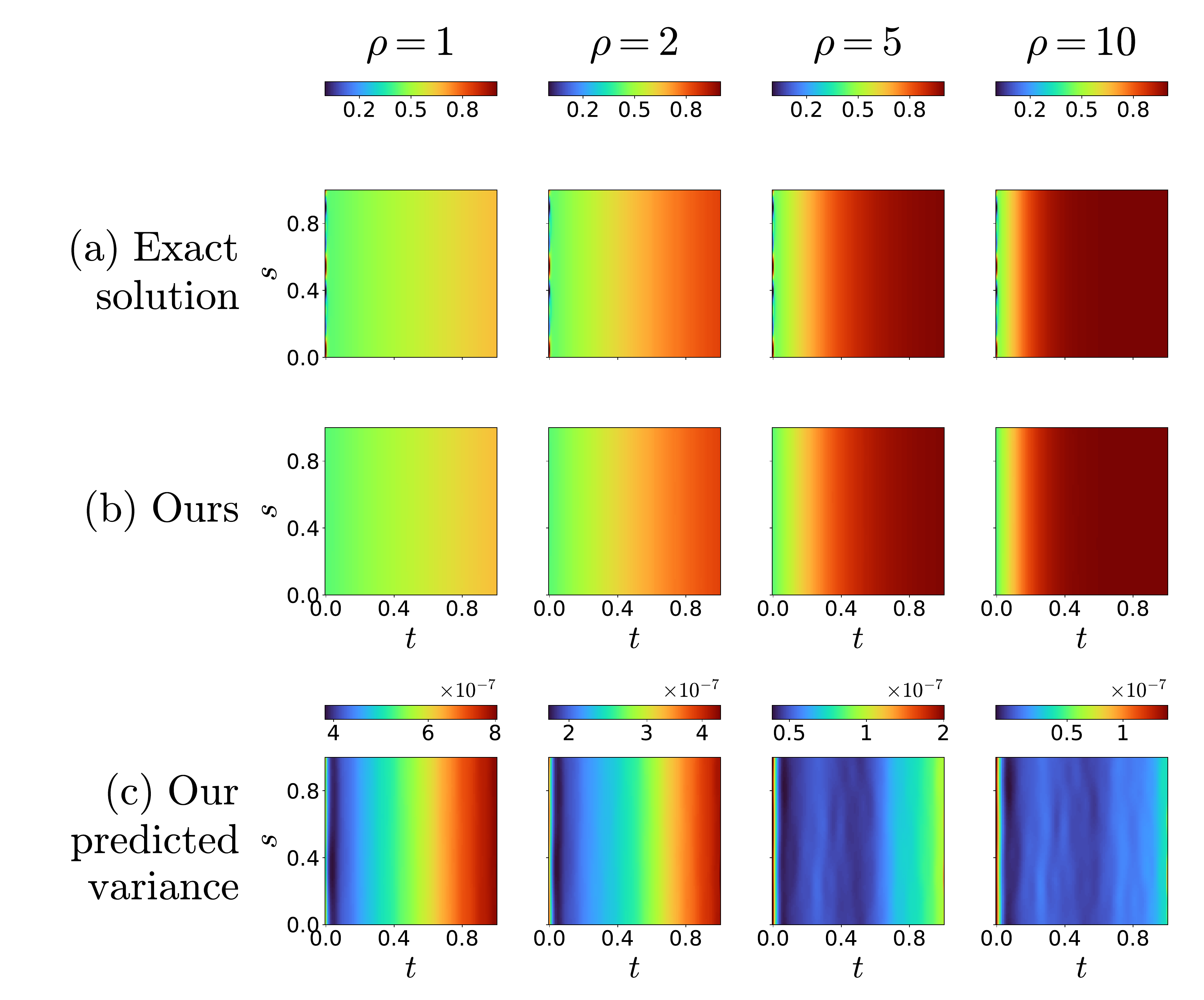}
    \caption{Our message passing approach for the Fisher-KPP equation on one PDEBench trajectory at $\nu=1$ across reaction strengths $\rho\in\{1,2,5,10\}$, in the coordinates $\vx=(s,t)$. \textbf{(a)} the exact solution; \textbf{(b)} predicted mean values of our message passing solution, which reproduce the reaction-driven saturation; \textbf{(c)} predicted variance values of our message passing solution.}
    \label{fig:fisher-pdebench}
\end{figure}

\begin{figure}
    \centering
    \includegraphics[width=0.86\linewidth]{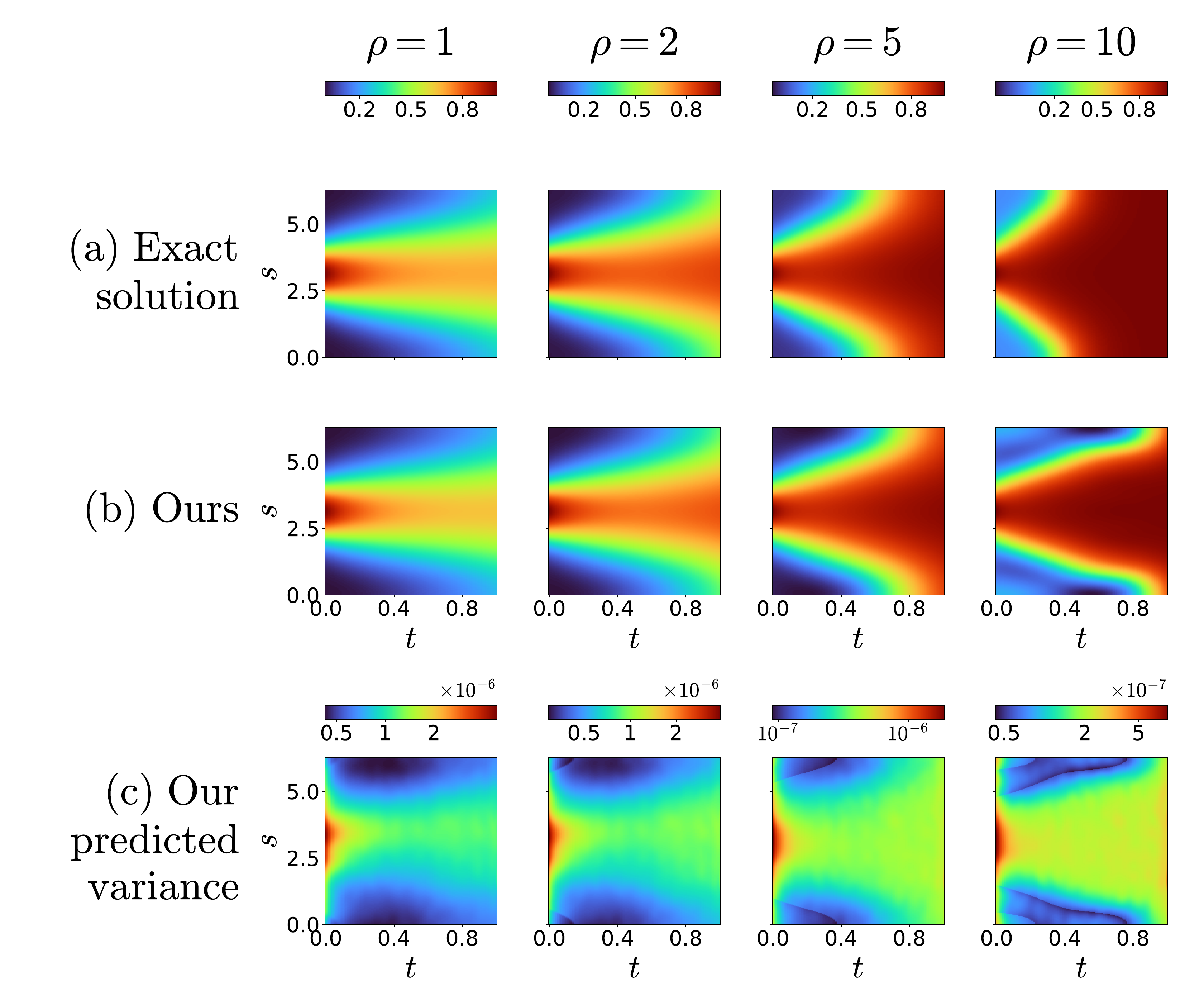}
    \caption{Our message passing approach for the Fisher-KPP equation with a Gaussian initial condition, $f(s,0)=\exp\!\big(-(s-\pi)^2/(2(\pi/4)^2)\big)$ on $s\in[0,2\pi)$, at $\nu=1$ across reaction strengths $\rho\in\{1,2,5,10\}$, in the coordinates $\vx=(s,t)$. \textbf{(a)} the exact solution; \textbf{(b)} predicted mean values of our message passing solution; \textbf{(c)} predicted variance values of our message passing solution.}
    \label{fig:fisher-simple}
\end{figure}

Figure \ref{fig:mp_vs_pdebench_scatter_grid_fisher} investigates a broader scope of the trajectories: it compares our approach with the non-Bayesian PINN baseline introduced in Section \ref{subsection:solution-fidelity} across different growth rates $\rho$ and demonstrates that both methods exhibit similar performance in terms of MSE and PDE residual, while our method even exhibits slightly lower residual.

\begin{figure}
    \centering
    \includegraphics[width=.95\linewidth]{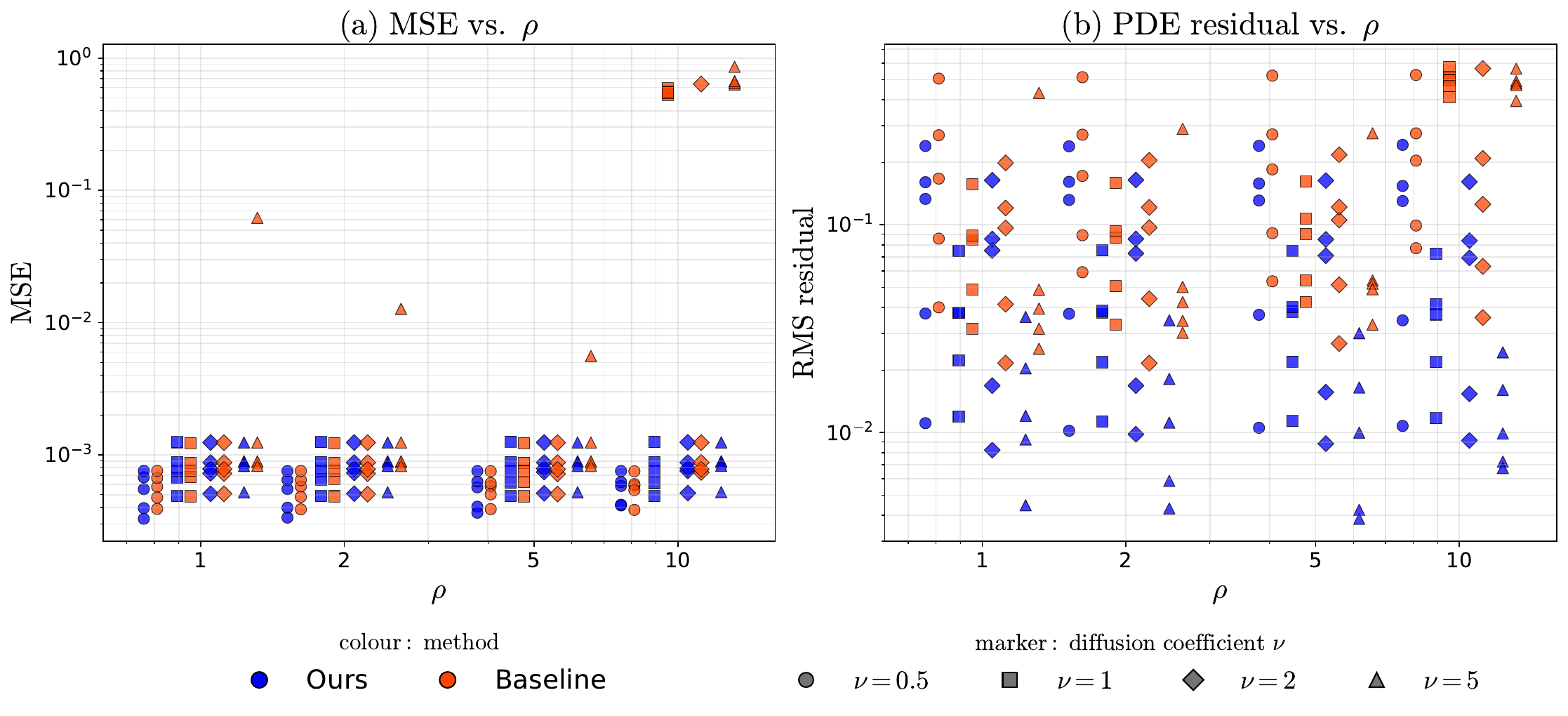}
    \caption{Comparison of our message passing approach (ours) and the SGD baseline for the Fisher-KPP equation across reaction strengths $\rho\in\{1,2,5,10\}$ (log scale), with the diffusion coefficient $\nu\in\{0.5,1,2,5\}$ shown by the marker, five trajectories per setting; each point refers to one trajectory. \textbf{(a)} MSE of the predictive mean against the exact solution; \textbf{(b)} root mean square (RMS) of the PDE residual $R$. Our errors are of the same order in all settings, while the baseline collapses to the trivial solution $f\equiv0$ on most runs at $\rho=10$, $\nu\geq1$.}
    \label{fig:mp_vs_pdebench_scatter_grid_fisher}
\end{figure}

Figure \ref{fig:bayesian_comparison_time_fisher} presents the time-efficiency results: our approach is significantly faster than both HMC and VI, and comparable to VI for inference. VI is the slowest during training because it does not meet the convergence criterion within the allocated training budget, as shown by the quantitative results in Table~\ref{tab:main_results}.

\begin{figure}
    \centering
    \includegraphics[width=1.\linewidth]{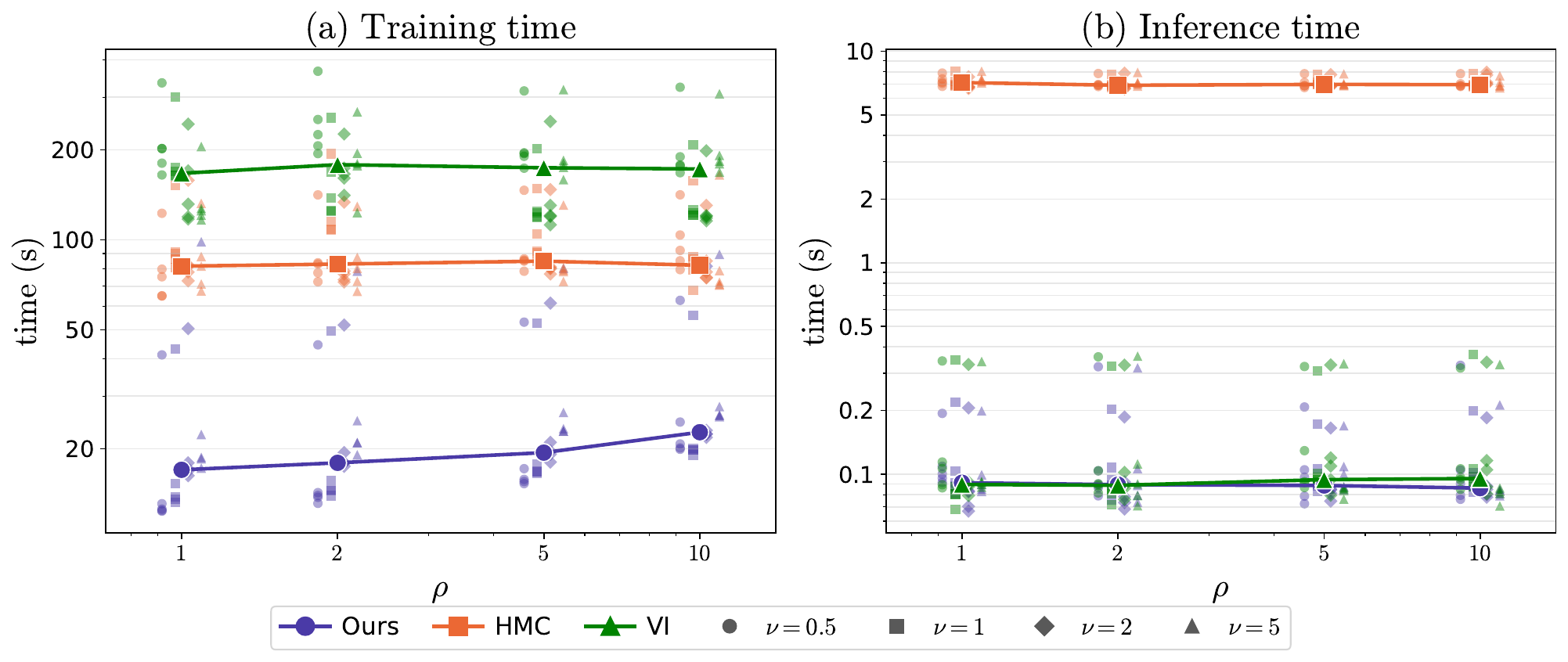}
    \caption{Computational cost of our approach (ours), HMC and VI for the Fisher-KPP equation across reaction strengths $\rho$, with the diffusion coefficient $\nu$ shown by the marker (points: runs; lines: medians). \textbf{(a)} Training time, cold start; \textbf{(b)} inference time of the predictive mean and variance on the full grid, steady state. Our training is about $4\times$ faster than HMC and $9\times$ faster than VI; our closed-form inference is as fast as VI's and about $80\times$ faster than HMC's sampling-based inference.}
    \label{fig:bayesian_comparison_time_fisher}
\end{figure}

Posterior comparison against HMC and VI is shown in Figures \ref{fig:bayesian_comparison_fisher} and \ref{fig:bayesian_comparison_ours_vi_fisher}. As with the advection results, we also observe that our approach infers means that are very close to those produced by HMC and VI, while our uncertainties are much closer to HMC uncertainties. Figure \ref{fig:bayesian_comparison_vi_hmc_fisher} additionally shows that VI standard deviations strongly decline from the HMC ones. Overall, it exhibits the same pattern as in the advection results, discussed in Section \ref{subsection:bayesian_comparison}.
\begin{figure}
    \centering
    \includegraphics[width=1.\linewidth]{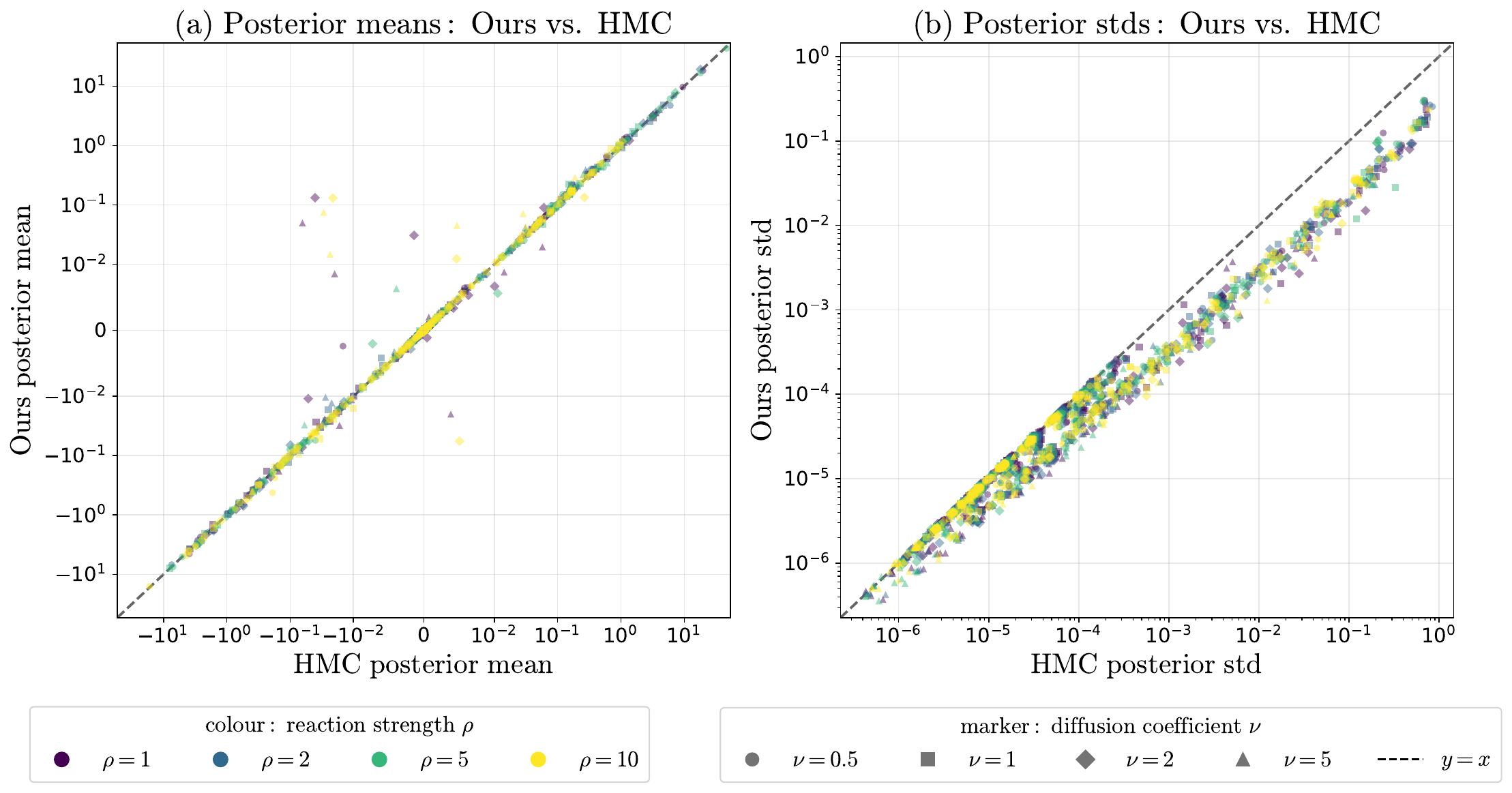}
    \caption{Comparison of our approach (ours) to HMC, implemented based on \cite{yang2021b}, for the Fisher-KPP equation (color: $\rho$; marker: $\nu$). We map posterior weights of HMC (x-axis) to our approach (y-axis) - points in \textbf{(a)} refer to the mean values, points in \textbf{(b)} refer to the standard deviations; the means agree closely, and our standard deviations are slightly smaller than, but proportional to, those of HMC. For readability, 25 random weights per trajectory are shown.}
    \label{fig:bayesian_comparison_fisher}
\end{figure}
\begin{figure}
    \centering
    \includegraphics[width=1.\linewidth]{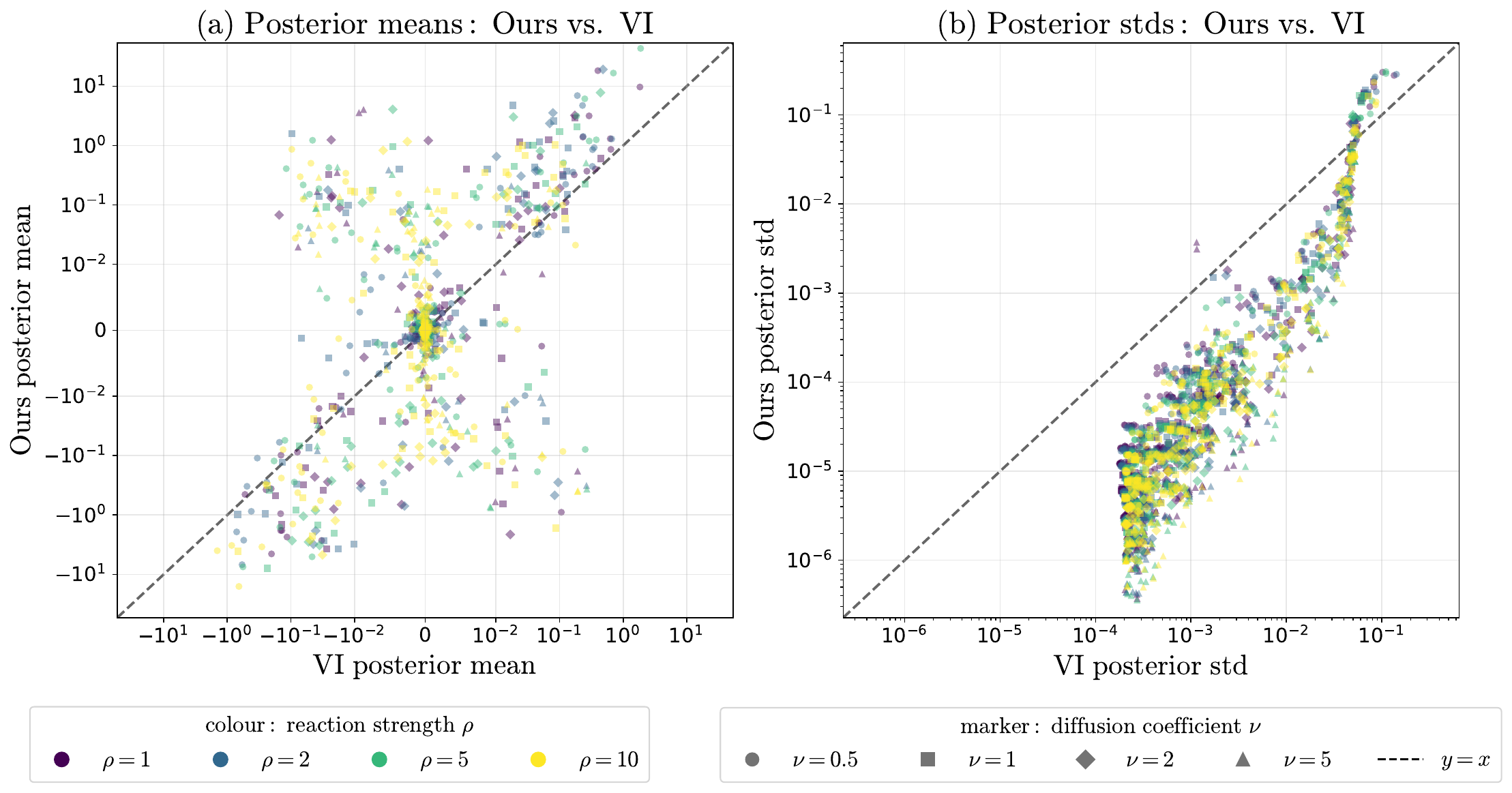}
    \caption{Comparison of our approach (ours) to VI, implemented based on \cite{yang2021b}, for the Fisher-KPP equation (colour: $\rho$; marker: $\nu$). We map posterior weights of VI (x-axis) to our approach (y-axis) - points in \textbf{(a)} refer to the mean values, points in \textbf{(b)} refer to the standard deviations; VI reaches neither our means nor our standard deviations, since it does not converge to the solution within the training budget. For readability, 25 random weights per trajectory are shown.}
    \label{fig:bayesian_comparison_ours_vi_fisher}
\end{figure}
\begin{figure}
    \centering
    \includegraphics[width=1.\linewidth]{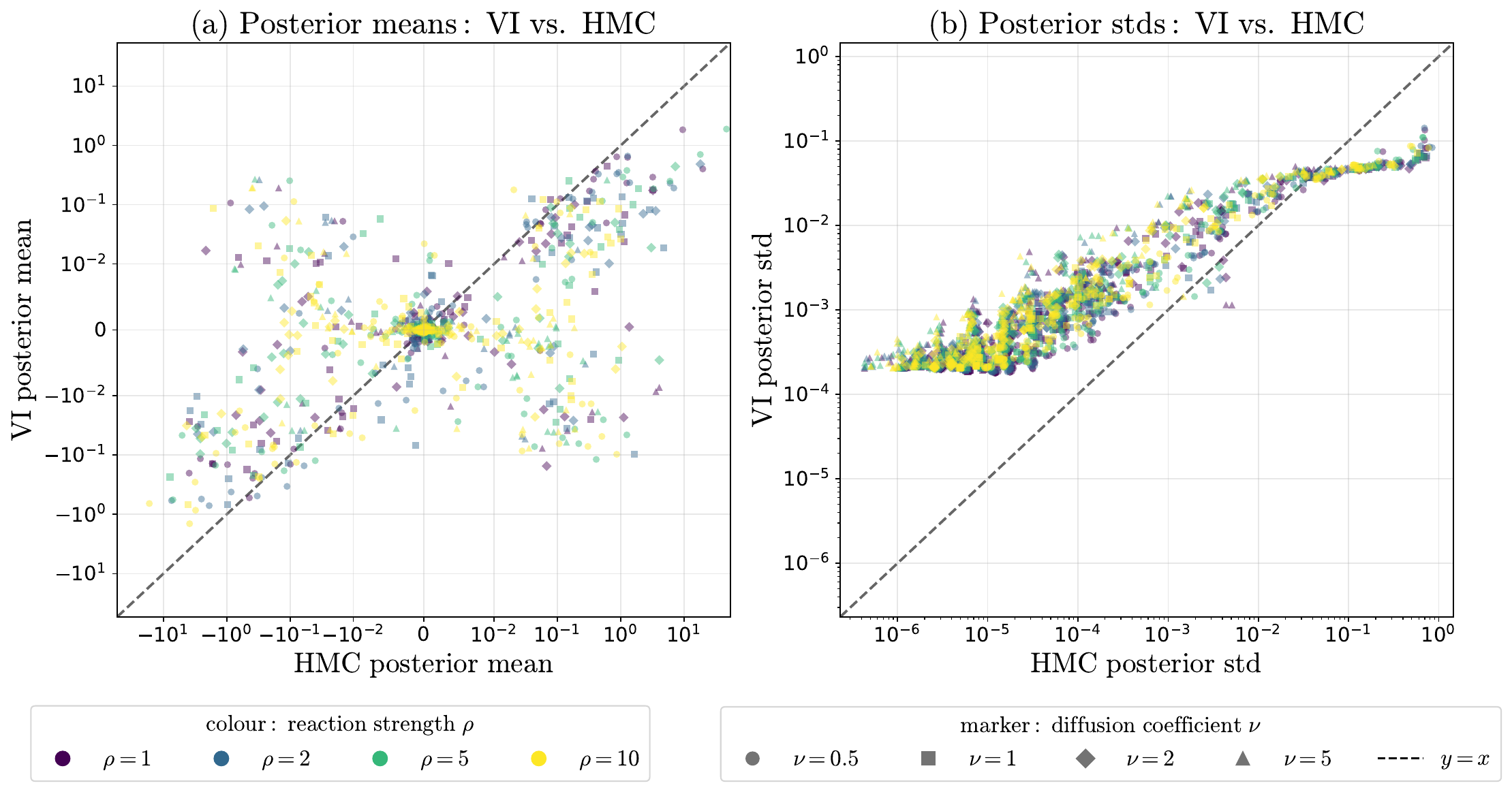}
    \caption{Comparison of VI to HMC, both implemented based on \cite{yang2021b}, for the Fisher-KPP equation (colour: $\rho$; marker: $\nu$). We map posterior weights of HMC (x-axis) to VI (y-axis) - points in \textbf{(a)} refer to the mean values, points in \textbf{(b)} refer to the standard deviations; VI neither recovers the HMC means nor its standard deviations, which it overestimates by up to two orders of magnitude. For readability, 25 random weights per trajectory are shown.}
    \label{fig:bayesian_comparison_vi_hmc_fisher}
\end{figure}

Figure \ref{fig:uq_fisher} compares the uncertainty calibration across these Bayesian methods. Unlike the advection results, the ratio of our predicted uncertainty to the HMC one is not as constant, though it does not change much. Calibration results also show similar patterns between our method and HMC.

\begin{figure}
    \centering
    \includegraphics[width=1.\linewidth]{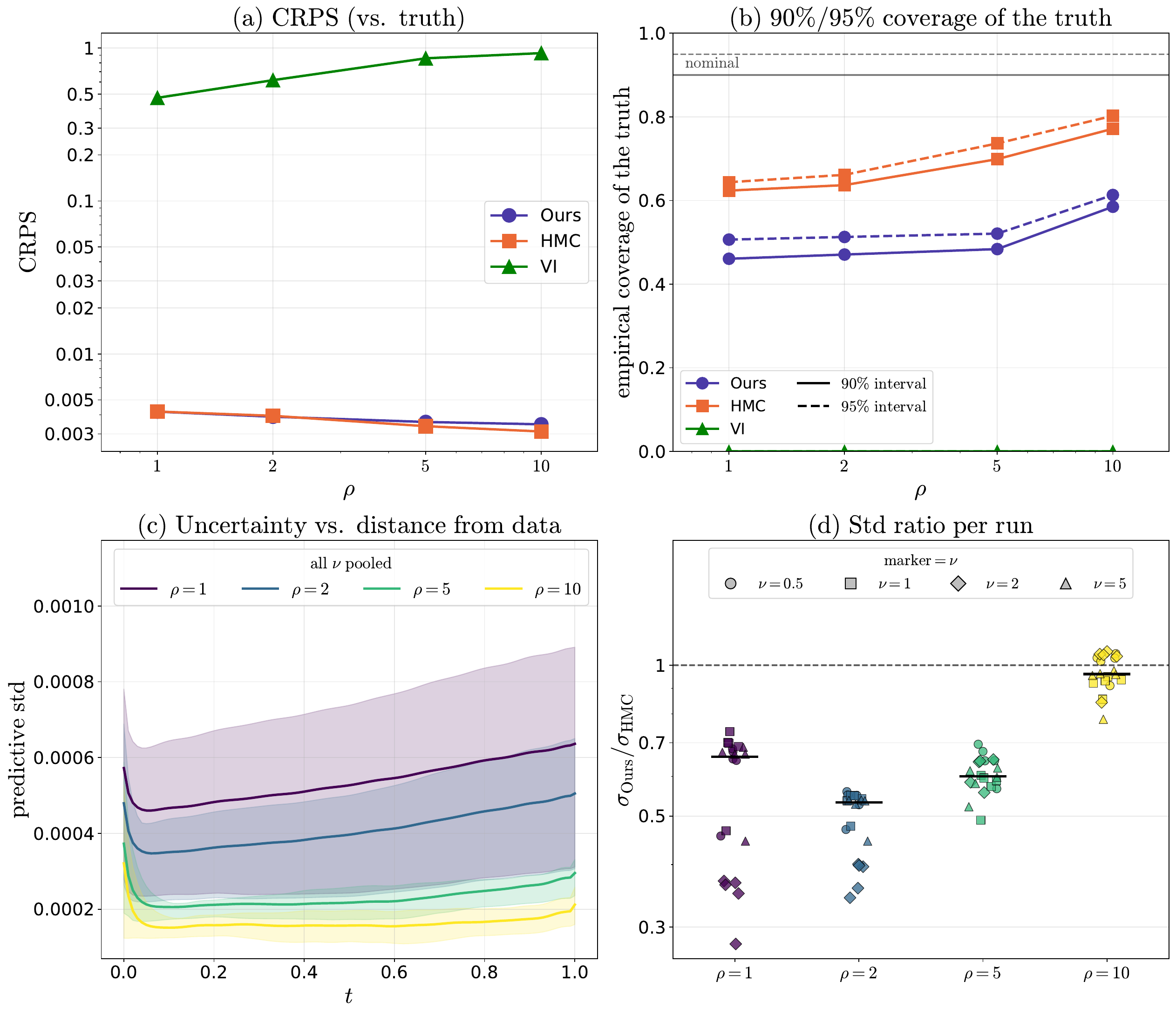}
    \caption{Calibration and uncertainty structure for the Fisher-KPP equation across reaction strengths $\rho$, pooled over the diffusion coefficients $\nu$. \textbf{(a)} CRPS against the exact solution; \textbf{(b)} empirical coverage of the exact solution by the nominal 90\,\% (solid) and 95\,\% (dashed) predictive intervals; \textbf{(c)} our mean predictive standard deviation over time $t$ (band: minimum to maximum over runs); \textbf{(d)} per-run median ratio of our predictive standard deviation to that of HMC (bar: median over runs; marker: $\nu$). Ours closely follows HMC in CRPS but is under-covered relative to it, with median ratios between 0.53 and 0.96; VI fails in all settings.}
\label{fig:uq_fisher}
\end{figure}


%

%

\subsubsection{PDE residual prior}
\label{appendix:pde-residual}
Figure~\ref{fig:mse_residual_vs_epsilon} studies the effect of the PDE tolerance \(\epsilon_C\) on the advection equation. Decreasing \(\epsilon_C\) enforces the PDE constraint more strongly and generally reduces both the prediction error and the PDE residual. However, when the constraint becomes too strong, especially for higher wave speeds (\(b\geq4\)), the solution can collapse towards a constant function. Such a solution satisfies the advection equation with a very small residual but does not preserve the propagated initial condition. Thus, \(\epsilon_C\) controls the balance between enforcing the PDE and retaining information from the observed initial condition.

\begin{figure}
    \centering
    \includegraphics[width=1.\linewidth]{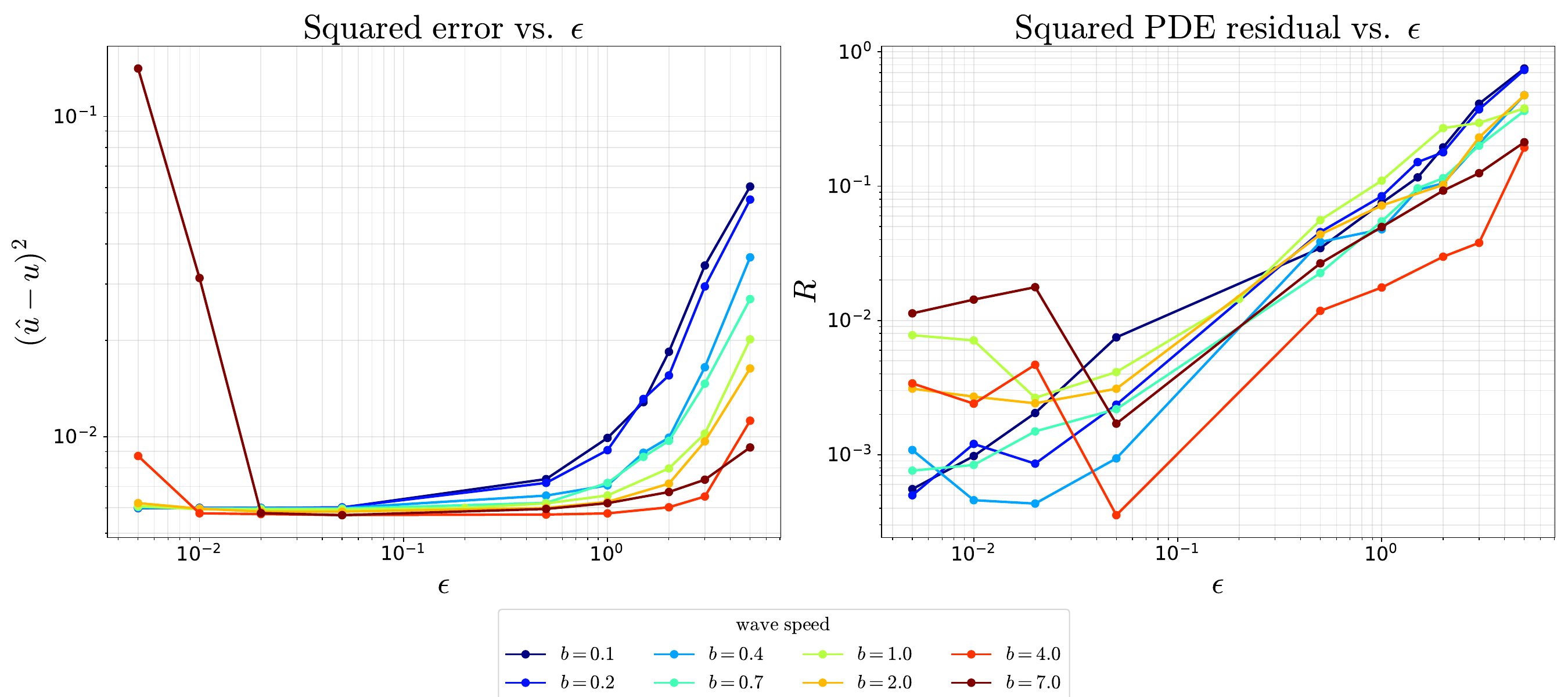}
    \caption{Effect of the PDE tolerance $\epsilon_C$ for the advection equation, ten trajectories per wave speed $b$ (lines: medians; bands: interquartile ranges; dotted: the default $\epsilon_C=0.05$). \textbf{(a)} MSE of the predictive mean against the exact solution; \textbf{(b)} mean squared PDE residual $R$. Tightening $\epsilon$ down to about $0.1$ lowers both; below that the error no longer improves, and at $b=7$ with $\epsilon=0.005$ the model collapses to a constant solution, i.e., exhibits a failure mode.}
    \label{fig:mse_residual_vs_epsilon}
\end{figure}

\end{document}